\documentclass[10pt]{article}

\usepackage[paperwidth=7in, paperheight=10in, textwidth=5.25in, textheight=8.2in]{geometry}

\usepackage{subfig} 
\usepackage{float}
\usepackage{amsmath}
\usepackage{amsthm}
\usepackage{amssymb}
\usepackage{natbib}
\usepackage{graphicx}
\usepackage{booktabs}
\usepackage{mathtools}
\usepackage{xr,tikz}
\usepackage{enumitem}
\usepackage{nicefrac}
\usepackage{forloop}
\usepackage[hidelinks]{hyperref}
\usepackage{url}
\usepackage{bbm} 
\usepackage{authblk}
\usepackage[capitalise]{cleveref}
\usepackage{environ}
\usepackage{bm,xspace}
\usepackage{microtype}
\newcommand{\euler}{e}
\newcommand{\Prob}{\mathbb{P}}
\newcommand{\E}{\mathbb{E}}
\newcommand{\Opt}{\textsc{OPT}\xspace}

\DeclareMathOperator*{\argmax}{arg\,max}
\newcommand{\Rep}{\textsc{Rep}\xspace}
\newcommand{\RepG}{\textsc{Rep-Greedy}\xspace}

\newcommand{\Tau}{\Gamma}
\newcommand{\AlgStreamOpt}{{\textsc{Replacement-Streaming-Know-Opt}}\xspace}
\newcommand{\AlgStream}{{\textsc{Replacement-Streaming}}\xspace}
\newcommand{\AlgStreamModified}{{\textsc{Replacement-Pseudo-Streaming}}\xspace}
\newcommand{\AlgRG}{{\textsc{Replacement-Greedy}}\xspace}
\newcommand{\AlgDistributed}{{\textsc{Replacement-Distributed}}\xspace}
\newcommand{\AlgDistributedFast}{{\textsc{Distributed-Fast}}\xspace}
\newcommand{\AlgExchange}{{\textsc{Exchange}}\xspace}
\newcommand{\AlgLocal}{{\textsc{LocalSearch}}\xspace}
\newtheorem{theorem}{Theorem}
\newtheorem{lemma}{Lemma}

\newtheorem{cor}{Corollary}

\usepackage[capitalise]{cleveref}
\usepackage{eqparbox,array}
\usepackage{algorithm}
\usepackage[noend]{algorithmic}

\let\emptyset\varnothing

\makeatother

\newcommand{\defcal}[1]{\expandafter\newcommand\csname c#1\endcsname{{\mathcal{#1}}}}
\newcommand{\defbb}[1]{\expandafter\newcommand\csname b#1\endcsname{{\mathbb{#1}}}}
\newcounter{calBbCounter}
\forLoop{1}{26}{calBbCounter}{
	\edef\letter{\Alph{calBbCounter}}
	\expandafter\defcal\letter
	\expandafter\defbb\letter
}

\title{Data Summarization at Scale:\\ A Two-Stage Submodular Approach}
\usepackage{times}

\author[1]{Marko Mitrovic \footnote{The first two authors have equally contributed to this paper.}}
\newcommand\CoAuthorMark{\footnotemark[\arabic{footnote}]} 
\author[1]{Ehsan Kazemi\protect\CoAuthorMark}
\author[2]{Morteza Zadimoghaddam}
\author[1]{\\Amin Karbasi}
\affil[1]{Yale Institute for Network Science\\Yale University}
\affil[2]{Google Research, Zurich, Switzerland}
\affil[ ]{\normalsize \texttt{\{marko.mitrovic, ehsan.kazemi, amin.karbasi\}@yale.edu, 
		zadim@google.com}}

\date{}
\begin{document}

\maketitle
\begin{abstract}
	The sheer scale of modern datasets has resulted in a dire need for summarization techniques that identify representative elements in a dataset.
Fortunately, the vast majority of data summarization tasks satisfy an intuitive diminishing returns condition known as submodularity, which allows us to find nearly-optimal solutions in linear time.
We focus on a two-stage submodular framework where the goal is to use some given training functions to reduce the ground set so that optimizing new functions (drawn from the same distribution) over the reduced set provides almost as much value as optimizing them over the entire ground set. In this paper, we develop the first streaming and distributed solutions to this problem. In addition to providing strong theoretical guarantees, we demonstrate both the utility and efficiency of our algorithms on real-world tasks including image summarization and ride-share optimization.

\end{abstract}

\section{Introduction}
In the context of machine learning, it is not uncommon to have to repeatedly optimize a set of functions that are fundamentally related to each other. In this paper, we focus on a class of functions called \textbf{submodular} functions. These functions exhibit a mathematical diminishing returns property that allows us to find nearly-optimal solutions in linear time. However, modern datasets are growing so large that even linear time solutions  can be computationally expensive. Ideally, we want to find a sublinear summary of the given dataset so that optimizing these related functions over this reduced subset is nearly as effective, but not nearly as expensive, as optimizing them over the full dataset.

As a concrete example, suppose Uber is trying to give their drivers suggested waiting locations across New York City based on historical rider pick-ups. Even if they discretize the potential waiting locations to just include points at which pick-ups have occurred in the past, there are still hundreds of thousands, if not millions, of locations to consider. If they wish to update these ideal waiting locations every day (or at any routine interval), it would be invaluable to be able to drastically reduce the number of locations that need to be evaluated, and still achieve nearly optimal results.

In this scenario, each day would have a different function that quantifies the value of a set of locations for that particular day. For example, in the winter months, spots near ice skating rinks would be highly valuable, while in the summer months, waterfront venues might be more prominent. On the other hand, major tourist destinations like Times Square will probably be busy year-round. 

In other words, although the most popular pick-up locations undoubtedly vary over time, there is also some underlying distribution of the user behavior that remains relatively constant and ties the various days together. This means that even though the functions for future days are technically unknown, if we can select a good reduced subset of candidate locations based on the functions derived from historical data, then this same reduced subset should perform well on future functions that we cannot explicitly see yet.

In more mathematical terms, consider some unknown distribution of functions $\mathbb{D}$ and a ground set $\Omega$ of $n$ elements to pick from. 
We want to select a subset $S$ of $\ell$ elements (with $\ell \ll n$) such that optimizing functions (drawn from distribution $\mathbb{D}$) over the reduced subset $S$ is comparable to optimizing them over the entire ground set $\Omega$.

This problem was first introduced by~\citet{balkanski2stage} as \textbf{two-stage submodular maximization}. This name comes from the idea that the overall framework can be viewed as two separate stages. First, we want to use the given functions to select a representative subset $S$, that is ideally sublinear in size of the entire ground set $\Omega$. In the second stage, for any functions drawn from this same distribution, we can optimize over $S$, which will be much faster than optimizing over $\Omega$. 

\textbf{Our Contributions.} In today's era of massive data, an algorithm is rarely practical if it is not scalable. In this paper, we build on existing work to provide solutions for two-stage submodular maximization in both the streaming and distributed settings. 
\cref{tab:results} summarizes the theoretical results of this paper and compares them with the previous state of the art.

\begin{table*}[t!]
	\centering
	\caption{Comparison of algorithms for two-stage monotone submodular maximization. Bounds that hold in expectation are marked (R). For distributed algorithms, we report the time complexity of each single machine, where $ \mathsf{M}$ represent the number of machines.} 
	\label{tab:results}
	\scalebox{0.7}{
		\begin{tabular}{lllll}
			\toprule
			\textbf{Algorithm} & \textbf{Approx.} & \textbf{Time Complexity}  & \textbf{Setup} &  \textbf{Function}\\
			\midrule
			\AlgLocal \citep{balkanski2stage} & $\nicefrac{1}{2}(1 - \nicefrac{1}{\euler})$ &  $O(km\ell n^2 \log n)$ & Centralized & Coverage functions only\\
			\AlgRG \citep{stan17} & $\nicefrac{1}{2}(1 - \nicefrac{1}{\euler^2})$ & $O(km\ell n)$ & Centralized & Submodular functions\\
			\bottomrule 
			\AlgStream (ours) & $\nicefrac{1}{7}$  & $O(k m n \log \ell )$  &  Streaming & Submodular functions\\
			\AlgDistributed (R) (ours) & $\nicefrac{1}{4}(1 - \nicefrac{1}{\euler^2})$ &  $O(\nicefrac{km\ell n}{\mathsf{M}} + \mathsf{M} km\ell^2)$  & Distributed & Submodular functions\\
			\AlgDistributedFast (R) (ours) & $0.107$  &  $O(\nicefrac{kmn \log \ell}{\mathsf{M}} + \mathsf{M} km\ell^2\log \ell )$ & Distributed & Submodular functions\\		
			\bottomrule
	\end{tabular}}
\end{table*}
\section{ Related Work} \label{section:related}
Data summarization is one of the most natural applications that falls under the umbrella of submodularity. As such, there are many existing works applying submodular theory to a variety of important summarization settings. For example, \citet{mirzasoleiman2013distributed} used an exemplar-based clustering approach to select representative images from the \textit{Tiny Images} dataset~\citep{torralba200880}. \citet{kirchhoff2014submodularity} and \citet{feldman2018do} also worked on submodular image summarization, while~\citet{lin2011class}  and \citet{wei2013using} focused on document summarization.

In addition to data summarization, submodularity appears in a wide variety of other machine learning applications including variable selection~\citep{krause05near}, recommender systems~\citep{GabillonKWEM2013}, crowd teaching~\citep{singla2014near}, neural network interpretability~\citep{elenberg17},  robust optimization~\citep{kazemi2017robust}, network monitoring~\citep{gomez10}, and influence maximization in social networks~\citep{kempe03}.

There have also been many successful efforts in scalable submodular optimization. For our distributed implementation we will primarily build on the framework developed by \citet{barbosa2015power}. Other similar algorithms include works by \citet{mirzasoleiman2013distributed} and \citet{mirrokni2015randomized}, as well as \citet{kumar13fast}. In terms of the streaming setting, there are two existing works we will focus on: \citet{badanidiyuru2014streaming} and  \citet{buchbinder2015online}. The key difference between the two
is that \citet{badanidiyuru2014streaming} relies on thresholding and will terminate as soon as $k$ elements are selected from the stream,  while \citet{buchbinder2015online} will continue through the end of the stream, swapping elements in and out when required.

Repeated optimization of related submodular functions has been a well-studied problem with works on structured prediction \citep{lin12mixtures}, submodular bandits \citep{yue11,chen17}, and online submodular optimization \citep{jegelka2011-online-submodular-min}. However, unlike our work, these approaches are not concerned with data summarization as a key pre-processing step.

The problem of two-stage submodular maximization was first introduced by \citet{balkanski2stage}. They present two algorithms with strong approximation guarantees, but both runtimes are prohibitively expensive. Recently, ~\citet{stan17} presented a new algorithm known as \AlgRG that improved the approximation guarantee from $\frac{1}{2}(1 - \frac{1}{\euler})$ to $\frac{1}{2}(1 - \frac{1}{\euler^2})$ and the run time from $O(km\ell n^2\log(n))$ to $O(km\ell n)$. 
They also show that, under mild conditions over the functions, maximizing over the sublinear summary can be arbitrarily close to maximizing over the entire ground set. In a nutshell, their method indirectly constructs the summary $S$ by greedily building up solutions $T_i$ for each given function $f_i$ simultaneously over $\ell$ rounds. 

Although \citet{balkanski2stage} and  \citet{stan17} presented centralized algorithms with constant factor approximation guarantees, there is a dire need for scalable solutions in order for the algorithm to be practically useful. 
 In particular, the primary purpose of two-stage submodular maximization is to tackle problems where the dataset is too large to be repeatedly optimized by simple greedy-based approaches. As a result, in many cases, the datasets can be so large that existing algorithms cannot even be run once. The greedy approach requires that the entire data must fit into main memory, which may not be possible, thus requiring a streaming-based solution. Furthermore, even if we have enough memory, the problem may simply be so large that it requires a distributed approach in order to run in any reasonable amount of time. 

\section{Problem Definition} \label{section:problem}
In general, if we want to optimally choose $\ell$ out of $n$ items, we need to consider every single one of the exponentially many possibilities. This makes the problem intractable for any reasonable number of elements, let alone the billions of elements that are common in modern datasets. Fortunately, many data summarization formulations satisfy an intuitive diminishing returns property known as  \textbf{submodularity}. 

More formally, a set function $f : 2^{V} \to \mathbb{R}$ is \textbf{submodular} ~\citep{fujishige91,krause12survey} if, for all sets $A \subseteq B \subseteq V$ and every element $v \in V \setminus B$, we have $f(A + v) - f(A) \geq f(B + v) - f(B)$.\footnote{For notational convenience, we use $A +v = A \cup \{v\}$.} That is, the marginal contribution of any element $v$ to the value of $f(A)$ diminishes as the set $A$ grows. 

Moreover, a submodular function $f$ is said to be \textbf{monotone} if $f(A) \leq f(B)$ for all sets $A \subseteq B \subseteq V$. That is, adding elements to a set cannot decrease its value. Thanks to a celebrated result by ~\citet{nemhauser78}, we know that if our function $f$ is monotone submodular, then the classical greedy algorithm will obtain a $(1 - 1/e)$-approximation to the optimal value. Therefore, we can nearly-optimize monotone submodular functions in linear time.

Now we formally re-state the problem we are going to solve.

\noindent
\textbf{Problem Statement.}  Consider some unknown distribution $\mathbb{D}$ of monotone submodular functions and a ground set $\Omega$ of $n$ elements to choose from. We want to select a set $S$ of at most $\ell$ items that maximizes the following function:
\begin{align}\label{eq:problem-dist}
G(S) = \E_{f \sim \mathbb{D}}[\max_{T \subseteq S, |T| \leq k} f(T)].
\end{align}
That is, the set $S$ we choose should be optimal in expectation over all functions in this distribution $\mathbb{D}$. However, in general, the distribution $\mathbb{D}$ is unknown and we only have access to a small set of functions $F = (f_1, \hdots, f_m)$  drawn from $\mathbb{D}$. Therefore, the best approximation we have is to optimize the following related function:
\begin{align} \label{eq:problem-m}
G_m(S) = \frac{1}{m} \sum_{i=1}^{m} \max_{T^*_i \subseteq S, |T^*_i| \leq k} f_i(T^*_i).
\end{align}
\citet[Theorem 1]{stan17} shows that with enough sample functions, $G_m(S)$ becomes an 
arbitrarily good approximation to $G(S)$.

To be clear, each $T^*_i \subset S$ is the corresponding size $k$  optimal solution for $f_i$. 
However, in general we cannot find the true optimal $T^*_i$, so throughout the paper we will use $T_i$ to denote the approximately-optimal size $k$ solution we select for each $f_i$. Table \ref{terms} (\cref{section:terms}) summarizes the important terminology and can be used as a reference, if needed.

It is very important to note that although each function $f_i$ is monotone submodular, $G(S)$ is \textbf{not} submodular \citep{balkanski2stage}, and thus using the regular greedy algorithm to directly build up $S$ will give no theoretical guarantees. We also note that although $G(S)$ is an instance of an XOS function \citep{feige09}, existing methods that use the XOS property would require an exponential number of evaluations in this scenario \citep{stan17}.
\section{Streaming Algorithm} \label{section:streaming}
In many applications, the data naturally arrives in a streaming fashion. This may be because the data is too large to fit in memory, or simply because the data is arriving faster than we can store it. Therefore, in the streaming setting we are shown one element at a time and we must immediately decide whether or not to keep this element. There is a limited number of elements we can hold at any one time and once an element is rejected it cannot be brought back. 

There are two general approaches for submodular maximization (under the cardinality constraint $k$) in the streaming setting: (i)  \citet{badanidiyuru2014streaming} introduced a thresholding-based framework where each element from the stream is added only if its marginal value is at least $\frac{1}{2k}$ of the optimum value.
The optimum is usually not known a priori, but they showed that with only a logarithmic increase in memory requirement, it is possible to efficiently guess the optimum value. 
(ii) \citet{buchbinder2015online} introduced streaming submodular maximization with preemption. 
At each step, they keep a solution $A$ of size $k$ with value $f(A)$. Each incoming element is added if and only if it can be exchanged with a current element of $A$ for a net gain of at least $\frac{f(A)}{k}$.
In this paper, we combine these two approaches in a novel and non trivial way in order to design a streaming algorithm (called \AlgStream) for the two-stage submodular maximization problem.

The goal of \AlgStream is to pick a set $S$ of at most $\ell$ elements from the data stream, where we keep sets $T_i \subseteq S,  1 \leq i \leq m $ as the solutions for functions $f_i$.
We continue to process elements until one of the two following conditions holds: (i) $\ell$ elements are chosen, or (ii) the data stream ends. 
This algorithm starts from empty sets $S$ and $\{T_i\}$. For every incoming element $u^t$, we use the subroutine \AlgExchange to determine whether we should keep that element or not. To formally describe \AlgExchange, we first need to define a few notations.

We define the marginal gain of adding an element $x$ to a set $A$ as follows:
$f_i(x| A)  =  f_i(x + A) - f_i(A).$
For an element $x$ and set $A$, $\Rep_i(x, A) $ is an element of $A$ such that removing it from $A$ and replacing it with $x$ results in the largest gain for function $f_i$, i.e.,
\begin{align} \label{eq:rep}
 \Rep_i(x, A) = \argmax_{y \in A} f_i(A+ x - y) - f_i(A). 
\end{align}
The value of this largest gain is represented by
\begin{align}\label{eq:delta}
\Delta_i(x, A) = f_i(A+ x - \Rep_i(x, A)) - f_i(A).
\end{align}
We define the gain of an element $x$ with respect to a set $A$ as follows:
\[ \nabla_i(x, A) = \left\{
\begin{array}{ll}
\mathbbm{1}_{\{f_i(x|A) \geq (\alpha / k) \cdot f_i(A) \}} f_i(x|A)& \mbox{if } |A| < k , \\
\mathbbm{1}_{\{\Delta_i(x, A) \geq (\alpha / k) \cdot f_i(A) \} } \Delta_i(x, A)& \mbox{o.w.,}
\end{array}
\right.\]
where $\mathbbm{1}$ is the indicator function. That is, $\nabla_i(x, A)$ tells us how much we can increase the value of $f_i(A)$ by either adding $x$ to $A$ (if $|A| < k$) or optimally swapping it in (if $|A| = k$). However, if this potential increase is less than $\frac{\alpha}{k} \cdot f_i(A)$, then $\nabla_i(x, A) = 0$. In other words, if the gain of an element does not pass a threshold of $\frac{\alpha}{k} \cdot f_i(A)$, we consider its contribution to be 0.

An incoming element is picked if the average of the $\nabla_i$ terms is larger than or equal to a threshold $\tau$.
Indeed, for $u^t$, the \AlgExchange routine computes the average gain $\frac{1}{m} \sum_{i=1}^{m} \nabla_i(u^t, T_i)$. If this average gain is at least $\tau$, then $u^t$ is added to $S$; $u^t$ is also added to all sets $T_i$ with $\nabla_i(u^t, T_i) > 0$.
\cref{alg:exchange} explains \AlgExchange in detail.

\begin{algorithm}[tb]
	\caption{\AlgExchange}
	\label{alg:exchange}
	\begin{algorithmic}[1]
		\STATE {\bfseries Input:} $u, S, \{T_i\},  \tau$ and $\alpha$ \hspace{0.2in}\COMMENT{$\nabla_i$ terms use $\alpha$}
		\IF{ $|S| < \ell$}
		\IF{$\frac{1}{m}\sum_{i=1}^{m} \nabla_i(u, T_i) \geq \tau$ 	}		 
		\STATE $S \gets S + u$
		\FOR {$1 \leq i \leq m$}
		\IF{$ \nabla_i(u, T_i)  > 0$}
		\IF{$|T_i| < k$}
		\STATE $T_i \gets T_i + u$
		\ELSE
		\STATE$T_i \gets T_i + u - \Rep(u, T_i)$
		\ENDIF
		\ENDIF
		\ENDFOR
		\ENDIF
		\ENDIF
	\end{algorithmic}
\end{algorithm}

Now we define the optimum solution to \cref{eq:problem-m} by
\begin{align*}
S^{m, \ell} = \argmax_{S \subseteq \Omega, |S| \leq \ell} \frac{1}{m} \sum_{i=1}^{m} \max_{|T| \leq k, T\subseteq S} f_i(T),
\end{align*}
where the optimum solution to each function is defined by
	\vspace{-5pt}
\begin{align*}
S^{m,\ell}_{i} = \argmax_{S \subseteq S^{m,\ell} , |S| \leq k} f_i(S).
\end{align*}
We define $\Opt = \frac{1}{m} \sum_{i=1}^{m}  f_i(S^{m,\ell}_i)$.

In \cref{section:know}, we assume that the value of \Opt is known a priori. This allows us to design \AlgStreamOpt, which has a constant factor approximation guarantee. Furthermore, in \cref{section:guess}, we show how we can efficiently guess the value of \Opt  by a moderate increase in the memory requirement. This enables us to finally explain \AlgStream.

\subsection{Knowing \Opt} \label{section:know}

If \Opt is somehow known a priori, we can use \AlgStreamOpt.
As shown in  \cref{alg:streaming_sieve}, we begin with empty sets $S$ and $\{T_i\}$. For each incoming element $u^t$, it uses \AlgExchange to update sets $S$ and $\{T_i\}$.
The threshold parameter $\tau$ in $\AlgExchange$ is set to $\frac{\Opt}{\beta \ell}$ for a constant value of $\beta$.
This threshold guarantees that if an element is added to $S$, then the average of functions $f_i$ over $T_i$'s is increased by a value of at least $\frac{\Opt}{\beta \ell}$.
Therefore, if we end up with $\ell$ elements in $S$, we guarantee that $\frac{1}{m} \sum_{i=1}^{m} f_i(T_i) \geq \frac{\Opt}{\beta}$. On the other hand, if $|S| < \ell$, we are still able to prove that our algorithm has picked good enough elements such that $\frac{1}{m} \sum_{i=1}^{m} f_i(T_i) \geq \frac{\alpha \cdot (\beta -1 ) \cdot \Opt}{\beta \cdot \left((\alpha+1)^2 + \alpha \right) } $.
The pseudocode of \AlgStreamOpt is provided in \cref{alg:streaming_sieve}. 

\begin{algorithm}[tb]
\caption{\AlgStreamOpt}
	\label{alg:streaming_sieve}
	\begin{algorithmic}[1]
		\STATE {\bfseries Input:} $\Opt$, $\alpha$ and $\beta$
		\STATE {\bfseries Output:} Sets $S$ and $\{T_i\}_{1 \leq i \leq m}$, where $T_i \subset S$
		\STATE $S \gets \emptyset$ and
		\STATE $T_i \gets \emptyset$ for all $1 \leq i \leq m$
		\FOR{every arriving element $u^t$}
		\STATE $\AlgExchange(u^t, S, \{T_i\},  \frac{\Opt}{\beta \ell}, \alpha )$
		\ENDFOR
		\STATE {\bfseries Return:} $S$ and $\{T_i\}_{1 \leq i \leq m}$
	\end{algorithmic}
\end{algorithm}

\begin{theorem} \label{theorem:streaming}
	The approximation factor of \AlgStreamOpt is at least $\min\{ \frac{\alpha (\beta -1 )}{\beta \cdot \left((\alpha+1)^2 + \alpha \right)} , \frac{1}{\beta}\}$. Hence, for $\alpha = 1$ and $\beta = 6$ the competitive ratio is at least $\nicefrac{1}{6}$.
\end{theorem}

\begin{proof} 
Let $S^t$ represent the set of chosen elements at step $t$. Also, we define $T^t_i \subseteq S^t$ as the current solution for function $f_i$ at step $t$. 
We also define $A^t_i= \bigcup_{1 \leq j \leq t} T^t_i$, i.e., $A^t_i$ is the set of all the elements have been in the set $T_i$ till step $t$. 
Note that this set includes elements that have been in $T_i$ at some point and might be deleted at later steps.
We first lower bound $f_i(T_i^t)$ based on value of $f_i(A_i^t)$.
\begin{lemma} \label{lemma:bound-S-A}
	For all $1 \leq i \leq m$, we have:
	\begin{align*}
	f_i(T^t_i) \geq \frac{\alpha}{\alpha+1} f_i(A_i^t).
	\end{align*}
\end{lemma}
\begin{proof}
	We proof this lemma by induction. 
	For the first $k$ additions to set $T_i^t$, the two sets $T_i^t$ and $A_i^t$ are exactly the same, i.e., we have $f_i(T_i^t) = f_i(A_i^t)$. Therefore the lemma is correct for them.  Next we show that lemma is correct for cases after the first $k$ additions, i.e., when an incoming element $u^t$ replaces one element of $T_i^{t-1}.$ We have the following lemma.
	\begin{lemma} \label{lemma:lowerbound-replace}
		For $1 \leq i \leq m$ and all $u^t$, we have:
		\begin{align*}
		\Delta_i(u^t, T_i^{t-1}) \geq f_i(u^t | A_i^{t-1}) - f_i(T_i^{t-1}) / k.
		\end{align*}
	\end{lemma}
	
	\begin{proof} To prove this lemma we have the following 
		\begin{align*}
		\Delta_i(u^t, T_i^{t-1})  &= f_i(T_i^{t-1}+ u^t - \Rep_i(u^t, T_i^{t-1})) - f_i(T_i^{t-1}) \\
		&\overset{(a)}{\geq} \dfrac{\sum_{u \in T_i^{t-1}} f_i(T_i^{t-1}+ u^t - u)) - f_i(T_i^{t-1}) }{k}  \\
		&  = \dfrac{\sum_{u \in T_i^{t-1}} f_i(T_i^{t-1}+ u^t - u) - f_i(T_i^{t-1} - u) +  f_i(T_i^{t-1}  - u) - f_i(T_i^{t-1}) }{k} \\ 
		&	\overset{(b)}{\geq}  \dfrac{\sum_{u \in T_i^{t-1}} f_i(T_i^{t-1}+ u^t ) - f_i(T_i^{t-1} ) }{k} + 
		\dfrac{\sum_{u \in T_i^{t-1}}   f_i(T_i^{t-1}  - u) - f_i(T_i^{t-1}) }{k} \\
		&	 \overset{(c)}{\ge}  f_i(u^t | T_i^{t-1}) - f_i(T_i^{t-1}) / k  \overset{(d)}{\geq}  f_i(u^t | A_i^{t-1}) - f_i(T_i^{t-1}) / k .
		\end{align*}
		Inequality $(a)$ is true because $\Rep_i(u^t, T_i^{t-1}) $ is the element with the largest increment when it is exchanged with $u^t$. 
		Therefore, it should be at least equal to the average of all possible exchanges. 
		Note that $T_i^{t-1}$ has at most $k$ elements. 
		Inequalities $(b)$ and $(d)$ result from submodularity of $f_i$.
		Also, from submodularity of $f_i$, we have 
		\[ f_i(T_i^{t-1}) - f_i(\emptyset) \geq \sum_{u \in T_i^{t-1}}  f_i(T_i^{t-1}) -   f_i(T_i^{t-1}  - u),\] 
		which results in inequality $(c)$.
	\end{proof}
	Now, assume \cref{lemma:bound-S-A} is true for time $t-1$, i.e.,  $f_i(T^{t-1}_i) \geq \frac{\alpha}{\alpha+1} f_i(A_i^{t-1})$. 
	We prove that it is also true for time $t$.
	First note that if $u^t$ is not accepted by the algorithm for the $i$-th function then $T_i^{t} = T_i^{t-1} $ and $A_i^{t} = A_i^{t-1}$; therefore the lemma is true for $t$.
	If $u^t$ is chosen to be added to $T_i^{t-1}$, from the definition of $\nabla(u^t, T_i^{t-1})$, we have $\Delta_i(u^t, T_i^{t-1}) > \alpha/ k \cdot f_i(T_i^{t-1}).$ 
	From this fact and \cref{lemma:lowerbound-replace}, we have:
	\begin{align*}
	f_i(T_i^t) - f_i(T_i^{t-1}) & \geq \max \{ f_i(u^t | A_i^{t-1}) - f_i(T_i^{t-1}) / k, \alpha/ k \cdot f_i(T_i^{t-1}) \}\\
	& \geq \dfrac{\alpha \cdot ( f_i(u^t | A_i^{t-1}) - f_i(T_i^{t-1}) / k) + \alpha/ k \cdot f_i(T_i^{t-1})}{\alpha + 1} \\
	& \geq \dfrac{\alpha}{\alpha+1} \cdot f_i(u^t | A_i^{t-1}) = \dfrac{\alpha}{\alpha+1}  \cdot\left[ f_i(A_i^{t}) - f_i(A_i^{t-1}) \right] \\
 &	\rightarrow   f_i(T_i^t) \geq  \dfrac{\alpha}{\alpha+1} \cdot f_i(A_i^{t}).
	\end{align*}
\end{proof}

\begin{cor} \label{cor:upperboun-margin}
	If $\Delta_i(u^t, T_i^{t-1}) < \alpha/ k \cdot f_i(T_i^{t-1})$ then we have:
	\[ f_i(u^t | A_i^{n}) \overset{(a)}{\leq} f_i(u^t | A_i^{t-1})  \overset{(b)}{\leq} \dfrac{\alpha + 1}{ k} \cdot f_i(T_i^{t-1})   \overset{(c)}{\leq} \dfrac{\alpha + 1}{ k}  \cdot f_i(T_i^{n}).\]
\end{cor}
\begin{proof}
	Inequality $(a)$ is true because of submodularity of $f_i$ and the fact that $A_i^{t-1} \subseteq\ A_i^n$.
	Inequality $(b)$ concludes form \cref{lemma:lowerbound-replace}. Since $f_i(T_i^{t})$ is a nondecreasing function of $t$, then $(c)$ is true.
\end{proof}

Next, we use \cref{lemma:bound-S-A,lemma:lowerbound-replace} and Corollary~\ref{cor:upperboun-margin}, to prove the approximation factor of the algorithm.
Note that if at the end of algorithm $|S^n| = \ell$, then we have:
\begin{align} \label{eq:ell-element}
\frac{1}{m}\sum_{i=1}^{m} f_i(T_i^n) = \frac{1}{m} \sum_{t=1}^{n} \sum_{i=1}^{m}  \left[ f_i(T_i^t) - f_i(T_i^{t-1}) \right]= \frac{1}{m} \sum_{t=1}^{n} \mathbbm{1}_{\{u^t \in S^n\}} \cdot \nabla_i(u^t, T_i^t)  \geq \frac{\Opt}{\beta}.
\end{align}
This is true because the additive value after adding an element to $S^t$ is at least $\frac{\Opt}{\beta \ell}$. 
Next consider the case where $|S| < \ell$.
First note that for an element $u^t \in S^{m,\ell}_i$,  which does not belong to set $A_i^n$, we have two different possibilities: 
(i) $\Delta_i(u^t, T_i^{t-1}) < \alpha / k \cdot f_i(T_i^{t-1}) $, or (ii) $\Delta_i(u^t, T_i^{t-1}) \geq \alpha / k \cdot f_i(T_i^{t-1}) $ and $\frac{1}{m}  \sum_{i=1}^{m} \nabla_i(u^t, T_i^{t-1}) < \frac{\Opt}{\beta \ell}.$ Therefore, we have:
\begin{align} \label{eq:approx-uppber}
& \sum_{i=1}^{m}  f_i(S_i^{m,\ell}) \leq \sum_{i=1}^{m} \left[ f_i(A^{n}_i) + \sum_{u^t \in S_i^{m,\ell} \setminus A^{n}_i} f_i(u^t | A^{n}_i) \right]  \nonumber \\
& =  \sum_{i=1}^{m}  f_i(A^{n}_i)  +  
\sum_{i=1}^{m} \sum_{u^t \in S^{m,\ell} \setminus A^{n}_i } \mathbbm{1}_{\{u^t \in S_i^{m,\ell} \} } \cdot f (u^t | A^{n}_i ) \nonumber \\
& =  \sum_{i=1}^{m}  f_i(A^{n}_i)  +  
\sum_{i=1}^{m} \sum_{u^t \in S^{m,\ell}} \mathbbm{1}_{\{u^t \in S_i^{m,\ell} \} } \cdot \nonumber 
 \Big[  \mathbbm{1}_{\{\Delta_i(u^t, T_i^{t-1}) < \alpha / k \cdot f_i(T_i^{t-1}) \} } \cdot f_i (u^t | A^{n}_i)  + \nonumber  \\
&  \hspace{1.24in}  \mathbbm{1}_{\{\Delta_i(u^t, T_i^{t-1}) \geq \alpha / k \cdot f_i(T_i^{t-1}) \text{ and } \sum_{i=1}^{m} \nabla_i(u^t, T_i^{t-1}) < \frac{\Opt}{\beta \ell}  \} } \cdot f_i (u^t | A^{n}_i) \Big].
\end{align}
For the three terms on the rightmost side of \cref{eq:approx-uppber} we have the following inequalities. For the first term, from \cref{lemma:bound-S-A}, we have:
\begin{align}  \label{eq:term-1}
\sum_{i=1}^{m}  f_i(A^{n}_i) \leq \dfrac{\alpha+1}{\alpha} \sum_{i=1}^{m}  f_i(T^{n}_i).
\end{align}
For the second term, we have:
\begin{align}  \label{eq:term-2}
\sum_{i=1}^{m} \sum_{u^t \in S^{m,\ell}}  \mathbbm{1}_{\{u^t \in S_i^{m,\ell} \} } \cdot &
\mathbbm{1}_{\{\Delta_i(u^t, T_i^{t-1}) < \alpha/ k \cdot f_i(T_i^{t-1}) \} } \cdot f_i (u^t | A^{n}_i) \nonumber \\
& \overset{(a)}{\leq} 	
\sum_{i=1}^{m} \sum_{u^t \in S^{m,\ell}_i}  \frac{\alpha + 1}{k} f_i(T_i^n)
\overset{(b)}{\leq} (\alpha + 1) \cdot \sum_{i=1}^{m} f_i(T_i^{n}).
\end{align}
Inequality $(a)$  is the result of Corollary~\ref{cor:upperboun-margin}. Inequality $(b)$ is true because we have at most $k$ elements in set $S_i^{m,\ell}.$
Note that for $u^t$ with $ \sum_{i=1}^{m} \nabla_i(u^t, T_i^{t-1}) < \frac{\Opt}{\beta \ell} $ we have:
\begin{align} \label{eq:bound-rejected}
 \frac{1}{m} \sum_{i=1}^{m}  \mathbbm{1}_{\{u^t \in S_i^{m,\ell} \} } & \cdot \mathbbm{1}_{\{\Delta_i(u^t, T_i^{t-1}) \geq \alpha / k \cdot f_i(T_i^{t-1})\} } \left[ f_i(u^t | A_i^{t-1}) - f_i(T_i^{t-1}) / k \right] \nonumber	\\
&\overset{(a)}{\leq} 
\frac{1}{m}  \sum_{i=1}^{m}  \mathbbm{1}_{\{u^t \in S_i^{m,\ell} \} } \nabla_i(u^t, T_i^{t-1})
   \overset{(b)}{\leq} \frac{1}{m}  \sum_{i=1}^{m} \nabla_i(u^t, T_i^{t-1}) < \frac{\Opt}{\beta \ell}.
\end{align}
Inequality $(a)$ results from \cref{lemma:lowerbound-replace} and $(b)$ is true because $\nabla_i(u^t, T_i^{t-1}) \geq 0$ for $1 \leq i \leq m$.
Therefore, from \cref{eq:bound-rejected} and submodularity of $f_i$ and its non-negativity, we have:
\begin{align*}
& \frac{1}{m} 	\sum_{i=1}^{m}  \mathbbm{1}_{\{u^t \in S_i^{m,\ell} \} } \cdot \mathbbm{1}_{\{\Delta_i(u^t, T_i^{t-1}) \geq \alpha/ k \cdot f_i(T_i^{t-1}) \text{ and } \sum_{i=1}^{m} \nabla_i(u^t, T_i^{t-1}) < \frac{\Opt}{\beta \ell}  \} } 	  \cdot f_i (u^t | A^{n}_i) \\  
& \hspace{2.9in}
\leq \dfrac{\Opt}{\beta \ell} +	\dfrac{1}{km} \cdot \sum_{i=1}^{m}  \mathbbm{1}_{\{u^t \in S_i^{m,\ell} \} } \cdot f_i(T_i^n).
\end{align*}
Consequently,
\begin{align} \label{eq:term-3}
& \frac{1}{m} 	\sum_{i=1}^{m} \sum_{u^t \in S^{m,\ell}}  \mathbbm{1}_{\{u^t \in S_i^{m,\ell} \} } \cdot \mathbbm{1}_{\{\Delta_i(u^t, T_i^{t-1}) \geq \alpha/ k \cdot f_i(T_i^{t-1}) \text{ and } \sum_{i=1}^{m} \nabla_i(u^t, T_i^{t-1}) < \frac{\Opt}{2 \ell}  \} }  \cdot f_i (u^t | A^{n}_i)  \nonumber \\
&\hspace{0.35in} \leq \frac{1}{m}  	\sum_{u^t \in S^{m,\ell}} \left[  \dfrac{\Opt}{\beta \ell} + 	\dfrac{1}{k} \cdot \sum_{i=1}^{m}  \mathbbm{1}_{\{u^t \in S_i^{m,\ell} \} } \cdot f_i(T_i^n) \right] \leq  \dfrac{\Opt}{\beta} +  \frac{1}{m}  \sum_{i=1}^{m}  f_i(T_i^n).
\end{align}
Using \cref{eq:term-1,eq:term-2,eq:term-3} we have:
\begin{align}
 \Opt & =\frac{1}{m} 	\sum_{i=1}^{m}  f_i(S_i^{m,\ell}) \nonumber \\
&  \leq \dfrac{\alpha +1}{\alpha} \cdot \frac{1}{m}  \sum_{i=1}^{m}  f_i(T^{n}_i) +  (\alpha + 1) \cdot \frac{1}{m}  \sum_{i=1}^{m} f_i(T_i^{n}) +  \dfrac{\Opt}{\beta} + \frac{1}{m}  \sum_{i=1}^{m}  f_i(T_i^n).
\end{align}
This results in 
\begin{align} \label{eq:less-than-ell}
\dfrac{\alpha \cdot (\beta - 1) \cdot \Opt}{\beta \cdot \left((\alpha +1)^2 + \alpha\right)}	\leq  \frac{1}{m}  \sum_{i=1}^{m} f_i(T_i^{n}).
\end{align}
Combination of \cref{eq:ell-element,eq:less-than-ell} proves \cref{theorem:streaming}.
\end{proof}

\subsection{Guessing \Opt in the Streaming Setting} \label{section:guess}

In this section, we discuss ideas on how to efficiently guess the value of $\Opt$, which is generally not known a priori.
First consider \cref{lemma:opt-bounds} where it provides bounds on $\Opt$.
\begin{lemma} \label{lemma:opt-bounds}
	Assume $\delta = \frac{1}{m} \max_{u \in \Omega} \sum_{i=1}^{m} f_i(u)$. Then we have
	$\delta \leq  \Opt \leq  \ell \cdot \delta. $
\end{lemma}

\begin{proof}
The lower bound is trivial.
For the upper bound we have
\[ \Opt = \frac{1}{m} \sum_{i=1}^{m} \sum_{u \in S^{m,\ell}_i}   f_i(u)  \leq  \frac{1}{m} \sum_{u \in S^{m,\ell}}\sum_{i=1}^{m}  f_i(u)  \leq \ell \cdot \delta.\]
\end{proof}

Now consider the following set
\[ \Tau = \{ (1 + \epsilon)^l \mid l \in \mathbb{Z}, \frac{\delta }{1 + \epsilon} \leq (1 + \epsilon)^l \leq \ell \cdot \delta\}. \]
We define $\tau_l = (1 + \epsilon)^l $.
From \cref{lemma:opt-bounds}, we know that one of the $\tau_l \in \Tau$ is a good estimate of \Opt. 
More formally, there exists a $\tau_l \in \Tau$ such that $\frac{ \Opt}{1 + \epsilon} \leq \tau_l \leq \Opt$.
For this reason, we should run parallel instances of \cref{alg:streaming_sieve}, one for each $\tau_l \in \Tau$. 
The number of such thresholds is $O(\frac{\log \ell}{\epsilon})$. 
 The final answer is the best solution obtained among all the instances.

Note that we do not know the value of $\delta$ in advance. 
So we would need to make one pass over the data to learn $\delta$, which is not possible in the streaming setting.
The question is, can we get a good enough estimate of $\delta$ within a single pass over the data?
Let's define 
$\delta^t = \frac{1}{m} \max_{u^{t'} , t' \leq t} \sum_{i=1}^{m} f_i(u^{t'})$ as our current guess for the maximum value of $\delta$.
Unfortunately, getting $\delta^t$ as an estimate of $\delta$ does not resolve the problem. This is due to the fact that a newly instantiated threshold $\tau$ could potentially have already seen elements with additive value of $\tau / (\beta \ell)$.
For this reason, we instantiate thresholds for an increased range of $\delta^t/(1 + \epsilon) \leq \tau_l \leq \ell \cdot \beta \cdot \delta^t$.
To show that this new range would solve the problem, first consider the next lemma.
\begin{lemma} \label{lemma:upperbound-gain}
	For the maximum gain of an incoming element $u^t$, we have:
	\[\frac{1}{m} \sum_{i=1}^{m} \nabla_i(u^t, T_i^{t-1}) \leq  \delta^t.\]
\end{lemma}
\begin{proof}
We have
\[ \frac{1}{m} \sum_{i=1}^{m} \nabla_i(u^t, T_i^{t-1}) \overset{(a)}{\leq} \frac{1}{m} \sum_{i=1}^{m} f_i(u^t | T_i^{t-1} ) \overset{(b)}{\leq}  \frac{1}{m} \sum_{i=1}^{m} f_i(u^t) \overset{(c)}{\leq} \delta_t. \]
For inequality $(a)$ first note that $f_i(u^t | T_{i}^{t-1}) \geq 0$; 
therefore it suffices to show that for all $\nabla_i(u^t, T_i^{t-1}) > 0$ we have $\nabla_i(u^t, T_i^{t-1}) \leq f_i(u^t| T_i^{t-1})$.
So, for $\nabla_i(u^t, T_i^{t-1}) > 0$, consider the two following cases: (i) if $|T_{i}^{t-1}| < k$, then $\nabla_i(u^t, T_i^{t-1}) = f_i(u^t| T_i^{t-1})$. (ii)  if $|T_{i}^{t-1}| < k$, then $\nabla_i(u^t, T_i^{t-1}) =  \Delta_i(u^t, T_i^{t-1}) = f_i(T_i^{t-1}+ u^t - \Rep_i(u^t, T_i^{t-1})) - f_i(T_i^{t-1}) \leq f_(T_i^{t-1} + u^t) -f_i(T_i^{t-1})$, where the last inequality follows from the monotonicity of $f_i$.  Inequality $(b)$ results from the submodularity of $f_i$. 
The inequality $(c)$ follows from the definition of $\delta_t$.
\end{proof}

Next, we need to show that for a newly instantiated threshold $\tau$ at time $t+1$, the gain of all elements which arrived before time $t+1$ is less than $\tau$; 
therefore this new instance of the algorithm would not have picked them if it was instantiated from the beginning. 
To prove this, note that since $\tau$ is a new threshold at time $t+1$, we have $\tau > \frac{\ell \cdot \beta \cdot \delta^t}{\beta \cdot \ell} = \delta^t$. From \cref{lemma:upperbound-gain} we conclude that the marginal gain of all the $u^{t'}, t' \leq t$ is less than $\tau$ and \AlgExchange would not have picked them.
The \AlgStream algorithm is shown pictorially in Figure \ref{fig:sieve-streaming} and the pseudocode  is given in \cref{alg:streaming_sieve_guess}. 

\begin{algorithm}[tb]
	\caption{\AlgStream}
		\label{alg:streaming_sieve_guess}
	\begin{algorithmic}[1]
		\STATE $\Tau^0 = \{(1 + \epsilon)^l | l \in \mathbb{Z} \}$
		\STATE For each $\tau \in \Tau^0$ set $S_{\tau} \gets \emptyset$ and $T_{\tau, i} \gets \emptyset$ for all $1 \leq i \leq m$ \hspace{0.1in} \COMMENT{Maintain the sets lazily}
		\STATE	$\delta^0 \gets 0$
		\FOR{every arriving element $u^t$}
		\STATE $\delta^t = \max \{ \delta^{t-1} ,  \frac{1}{m} \sum_{i=1}^{m} f_i(u^t) \}$
		\STATE $\Tau^t = \{ (1 + \epsilon)^l \mid l \in \mathbb{Z}, \frac{\delta^t }{(1 + \epsilon) \cdot \beta \cdot \ell} \leq (1 + \epsilon)^l \leq  \delta^t \}$
		\STATE	Delete all $S_\tau$ and $T_{\tau,i}$ such that $\tau \notin \Tau^t$
		\FOR{all $\tau \in \Tau^t$}
		\STATE $\AlgExchange(u^t, S_\tau, \{T_{\tau,i}\}_{1 \leq i \leq m}, \tau, \alpha )$
		\ENDFOR 
		\ENDFOR 
		\STATE {\bfseries Return:} $\argmax_{\tau \in \Tau^n} \{ \frac{1}{m} \sum_{i=1}^{m} f_i(T_{\tau, i}) \}$ 
	\end{algorithmic}
\end{algorithm}

\begin{figure}[tp]
	\centering
	\includegraphics[width=0.75\textwidth]{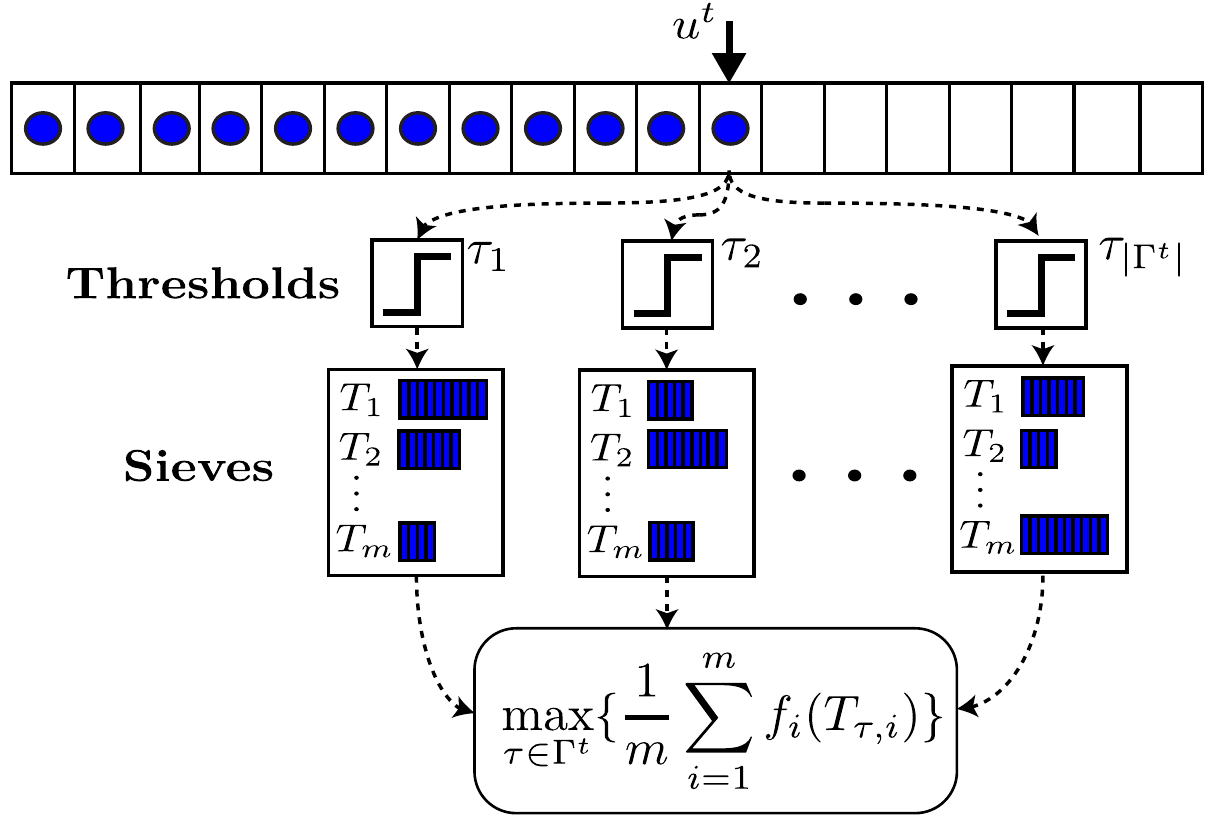}
	\caption{Illustration of \AlgStream.  Stream of data arrives at any arbitrary order. At each step $t$, the set of thresholds $\Tau^t$ is updated based on a new estimation of $\delta^t$. Note that at each time the number of such thresholds is bounded. For each $\tau \in \Tau^t$ there is a running instance of the streaming algorithm.}
	\label{fig:sieve-streaming}
\end{figure}

\begin{theorem} \label{theorem:streaming_guess}
	\cref{alg:streaming_sieve_guess} satisfies the following properties:
	\begin{itemize}
	\item It outputs sets $S$ and $\{T_i\}\subset S$ for $1 \leq i \leq m$, such that $|S| \leq \ell, |T_i| \leq k$ and $\frac{1}{m} \sum_{i=1}^{m} f_i(T_i) \geq \min\{ \frac{\alpha (\beta -1 )}{\beta  \left((\alpha+1)^2 + \alpha \right)} , \frac{1}{\beta(1 + \epsilon)}\}  \cdot \Opt$.
	\item  For $\alpha = 1$ and $\beta =\frac{6 + \epsilon}{1 + \epsilon}$ the approximation factor is at least $\frac{1}{6 + \epsilon}$. For $\epsilon = 1.0 $ the approximation factor is $\nicefrac{1}{7}$.
	\item It makes one pass over the dataset and stores at most $O(\frac{\ell \log \ell}{\epsilon})$ elements. The update time per each element is $O(\frac{ k m \log \ell}{\epsilon})$.
	\end{itemize}
\end{theorem}

\begin{proof}
Note that there exists an instance of algorithm with a threshold $\tau$ in $\Tau^n$ such that $ \frac{\Opt}{1 + \epsilon} \leq \tau_l \leq \Opt$. 
For this instance, it suffices to replace $\Opt$ with $\frac{\Opt}{1 + \epsilon }$ in the proof of \cref{theorem:streaming}.
This proves the approximation guarantee of the theorem.
For each instance of the algorithm we keep at most $\ell$ items. 
Since we have $O(\frac{\log \ell}{\epsilon})$ thresholds, the total memory complexity of the algorithm is $O(\frac{ \ell \log \ell}{\epsilon})$. The update time per each element $u^t$ for each instance is $O(km)$. 
This is true because we compute the gain of exchanging $u^t$ with all the $k$ elements of $T_i^{t-1}$ for each function $f_i, 1\leq i \leq m$.
Therefore, the total update time per elements is $O(\frac{ k m \log \ell}{\epsilon})$.
\end{proof}

\section{Distributed Algorithm} \label{section:distributed}

In recent years, there have been several successful approaches to the problem of distributed submodular maximization
\citep{kumar13fast,mirzasoleiman2013distributed,mirrokni2015randomized,barbosa2015power}.
Specifically,
\citet{barbosa2015power} proved that the following simple procedure results in a distributed algorithm with a constant factor approximation guarantee: (i) randomly split the data amongst $\mathsf{M}$ machines, (ii) run the classical greedy on each machine and pass outputs to a central machine, (iii) run another instance of the greedy algorithm over the union of all the collected outputs from all $\mathsf{M}$ machines, and (iv) output the maximizing set amongst all the collected solutions. Although our objective function $G(S)$ is \textbf{not} submodular, we use a similar framework and still manage to prove that our algorithms achieve constant factor approximations to the optimal solution.

In \AlgDistributed (\cref{alg:distributed}), a central machine first randomly partitions data among $\mathsf{M}$ machines. Next, each machine runs \AlgRG \citep{stan17} on its assigned data. The outputs $S^l, \{T_i^l\}$ of all the machines are sent to the central machine, which runs another instance of \AlgRG over the union of all the received answers. Finally, the highest value set amongst all collected solutions is returned as the final answer. 
See \cref{section-replace-greedy} for a detailed explanation of \AlgRG.

\begin{algorithm}[htb!]
	\caption{\AlgDistributed}
	\label{alg:distributed}
	\begin{algorithmic}[1]
		\FOR{$e \in \Omega$}
		\STATE	Assign $e$ to a machine chosen uniformly at random
		\ENDFOR
		\STATE	Run $\AlgRG$ on each machine $l$ to obtain $S^l$ and $\{T^l_i\}$ for $1 \leq i \leq m$
		\STATE	$S, \{T_i\} \gets \argmax_{S^l, \{T_i^{l}\} } \frac{1}{m} \sum_{i=1}^{m} f_i(T^l_i)$
		\STATE	$S', \{T'_i\} \gets \AlgRG(\bigcup_l S^l)$
		\STATE {\bfseries Return:} $\argmax\{\frac{1}{m} \sum_{i=1}^{m} f_i(T_i), \frac{1}{m} \sum_{i=1}^{m} f_i(T'_i)\}$
	\end{algorithmic}
\end{algorithm}

\begin{theorem}\label{theorem:distributed}
	The \AlgDistributed algorithm outputs sets $S^*, \{ T^*_i\}\subset S$, with $|S^*| \leq \ell, |T^*_i| \leq k$, such that
		\[ 
		\E[\frac{1}{m} \sum_{i=1}^m f_i(T^*_i)] \geq \frac{\alpha}{2} \cdot \Opt,
		\]
		where $\alpha = \frac{1}{2} (1 - \frac{1}{\euler^2}).$ The time complexity of algorithm is $O(\nicefrac{km\ell n}{\mathsf{M}} + \mathsf{M} km\ell^2)$.
\end{theorem}

\begin{proof}
First recall that we defined:
\begin{align*}
S^{m, \ell} = \argmax_{S \subseteq \Omega, |S| \leq \ell} \frac{1}{m} \sum_{i=1}^{m} \max_{|T| \leq k, T\subseteq S} f_i(T),
\end{align*}
and 
\begin{align*}
S^{m,\ell}_{i} = \argmax_{S \subseteq S^{m,\ell} , |S| \leq k} f_i(S) \text{ and } \Opt = \frac{1}{m}\sum_{i=1}^{m} f_i(S_i^{m,\ell}).
\end{align*}
Let $\mathcal{V}(\nicefrac{1}{\mathsf{M}})$ denote the distribution over random subsets of $\Omega$ where each element is picked independently with a probability $\frac{1}{\mathsf{M}}$.
Define vector $\bm{p} \in [0,1]^n$ such that for $e \in \Omega$, we have \[ \bm{p}_e=\left\{
\begin{array}{ll}
\Prob_{A \sim \mathcal{V}(1/\mathsf{M})} [e \in \AlgRG(A \cup \{e\})]\  \text{if } e \in S^{m,\ell},  \\
0 \qquad \text{otherwise.}
\end{array}
\right.
\]
We also define vector $\bm{p_i}$ such that for $e \in V,$ we have:
\[ {\bm{p_i}}_e=\left\{
\begin{array}{ll}
\bm{p}_e \  \text{if } e \in S^{m,\ell}_i,  \\
0 \qquad \text{otherwise.}
\end{array}
\right.
\]
Denote by $V^l$ the set of elements assigned to machine $l$.
Also, let $O^{l} = \{ e \in S^{m, \ell}: e \notin \AlgRG(V^l \cup \{e\}) \}$. 
Furthermore, define $O^{l}_i = O^{l} \cap S^{m,\ell}_i.$
The next lemma plays a crucial role in proving the approximation guarantee of our algorithm. 
\begin{lemma} \label{lemma:greedy}
	Let $A \subseteq \Omega$ and $B \subseteq \Omega$ be two disjoint subsets of $\Omega$. Suppose for each element $e \in B$, we have $\AlgRG(A \cup \{e\}) = \AlgRG(A)$. Then we have:
	\[\AlgRG(A \cup B) = \AlgRG(A).\]
\end{lemma}
\begin{proof}
	We proof lemma by contradiction. Assume
	\[\AlgRG(A \cup B) \neq \AlgRG(A).\] 
	At each iteration the element with the highest additive value is added to set $S$. 
	In \AlgRG, the additive value of each element depends on sets $T_i \subseteq S$.
	Note that sets $T_i \subseteq S$ are deterministic functions of elements of $S$ while considering their order of additions to $S$.
	Let's assume $e$ is the first element such that $\AlgRG(A \cup B) \neq \AlgRG(A)$. First note that $e \notin A$. Also, we conclude 
	\[\AlgRG(A \cup \{e\}) \neq \AlgRG(A).\] 
	This contradicts with the assumption of lemma.
\end{proof}

From the definition of set $O^l$ and \cref{lemma:greedy},  we have:
\[\AlgRG(V^l) = \AlgRG(V^l \cup O^l).\]
\begin{lemma}
	We have:
	\[\frac{1}{m} \sum_{i=1}^m f_i(T^{l}_{i}) \geq \alpha \cdot \frac{1}{m} \sum_{i=1}^m f_i(O^l_i), \]
	where $\alpha$ is the approximation factor of \AlgRG.
\end{lemma}
\begin{proof}
	Let $\Opt_i^l$ denote the optimum value for function $f_i$ on the dataset $V^l \cup O^l$ for the two-stage submodular maximization problem. We have:
	\[\frac{1}{m} \sum_{i=1}^m f_i(T^{l}_{i})\geq \alpha \cdot \frac{1}{m} \sum_{i=1}^m \Opt_i^l \geq \alpha \cdot \frac{1}{m} \sum_{i=1}^m f_i(O^l_i). \]
\end{proof}
This is true because (i) $\AlgRG(V^l) = \AlgRG(V^l \cup O^l)$, (ii) approximation guarantee of $\AlgRG$ is $\alpha$, and
(iii) $O^l$ and $\{O^l_i\}$ is a valid solution for the two-stage submodular maximization problem over set $V^l \cup O^l$. 
Assume $f_i^{-}$ is the Lov\'asz extension of a submodular function $f_i$.
\begin{lemma}[Lemma~1, \citet{barbosa2015power}]
	Let $A$ be random set, and suppose that $\E[\bm{1}_A] = \lambda \cdot \bm{p}$ for a constant value of $\lambda \in [0,1]$. Then,  $\E[f(S)] \geq \lambda \cdot f^{-}(\bm{p})$.
\end{lemma}

For each element $e \in S^{m, \ell}$ we have: 
\begin{align*}
\Prob[e \in O^l] & = 1 - \Prob[e \notin O^l] = 1 - \bm{p}_e,\\
\E[\bm{1}_{O^l}] & = \bm{1}_{S^{m,\ell}} - \bm{p},  \\
\E[\bm{1}_{O^l_i}] & = \bm{1}_{S^{m,\ell}_i} - \bm{p_i}.
\end{align*}
Therefore, we have:
\[ \E[ \frac{1}{m} \sum_{i=1}^m f_i(T^{l}_{i})] \geq \alpha \cdot  \E[ \frac{1}{m} \sum_{i=1}^m f_i(O^l_i) ] \geq \frac{\alpha}{m} \cdot \sum_{i=1}^{m} f^{-}_i (\bm{1}_{S^{m,\ell}_i} - \bm{p_i}).\]
Furthermore, for each element $e \in S^{m,\ell}$ we have
\begin{align*}
\Prob[e \in \bigcup_l S^l | e \text{ is assigned to machine } l] & = \Prob[e \in \AlgRG(V^l) | e \in V^l] \\
& = \Prob_{A \sim \mathcal{V}(1/\mathsf{M})} [e \in \AlgRG(A) | e \in A] \\
& = \Prob_{B \sim \mathcal{V}(1/\mathsf{M})} [e \in \AlgRG(B \cup \{e\})] \\
& = \bm{p}_e.
\end{align*}
Therefore, we have
\begin{align*}
\E[\frac{1}{m} \sum_{i=1}^{m} f_j(T'_i)] \geq \alpha \cdot  \E[\frac{1}{m} \sum_{i=1}^{m} f_i(\bigcup_l S^l \cap S^{m,\ell}_i) ] \geq \frac{\alpha}{m} \cdot \sum_{i=1}^{m} f^{-}_i(\bm{p_i})
\end{align*}
To Sum up above, we have:
\begin{align}
\E[ \frac{1}{m} \sum_{i=1}^{m} f_j(T^*_i)] & \geq \frac{\alpha}{m} \sum_{i=1}^{m} f^{-}_j (\bm{1}_{S^{m,\ell}_i} - \bm{p_i}), \\
\E[ \frac{1}{m} \sum_{i=1}^{m} f_i(T^*_i)] & \geq  \frac{\alpha}{m} \sum_{i=1}^{m} f^{-}_i(\bm{p_i}).
\end{align}
And therefore we have:
\begin{align*}
\E[ \frac{1}{m} \sum_{i=1}^{m} f_i(T^*_i)] & \geq 
 \frac{\alpha}{2m} \sum_{i=1}^{m} \left[ f^{-}_i(\bm{p_i}) + f^{-}_i (\bm{1}_{S^{m,\ell}_i} - \bm{p_i})  \right] \\
& \overset{(a)}{\geq} \frac{\alpha}{2m} \sum_{i=1}^{m}  f^{-}_i (\bm{1}_{S^{m,\ell}_i})  
\geq \frac{\alpha}{2m} \sum_{i=1}^m f_i(S^{m,\ell}_i).
\end{align*}
The inequality $(a)$ results from the convexity of Lov\'asz extensions for submodular functions. Note that the approximation guarantee of $\AlgRG$ is $\alpha =  \frac{1}{2} (1 - \frac{1}{\euler^2})$ \citep{stan17}. 
This proves \cref{theorem:distributed}.
\end{proof}

Unfortunately, for very large datasets, the time complexity of \AlgRG could be still prohibitive.
For this reason, we can use a modified version of \AlgStream (called \AlgStreamModified) to design an even more scalable distributed algorithm.
This algorithm receives all elements in a centralized way, but it uses a predefined order to generate a (pseudo) stream before processing the data.
This consistent ordering is used to ensure that the output of \AlgStreamModified is independent of the random ordering of the elements.
The only other difference between \AlgStreamModified and \AlgStream is that 
it outputs all sets $S_\tau, \{T_{\tau,i}\}$ for all $\tau \in \Tau^n$ (instead of just the maximum). 
We use this modified algorithm as one of the main building blocks for \AlgDistributedFast (outlined in \cref{alg:distributed-fast}).

\begin{algorithm}[htb!]
	\caption{\AlgDistributedFast}
	\label{alg:distributed-fast}
	\begin{algorithmic}[1]
		\STATE For $1 \leq l \leq \mathsf{M}$ set $V^l = \emptyset$
		\FOR{$e \in \Omega$}
		\STATE	Assign $e$ to a set $V^l$ chosen uniformly at random
		\ENDFOR
		\STATE For $1 \leq l \leq \mathsf{M}$ sort elements of $V^l$ based on a universal predefined ordering between elements \hspace{0.1in} \COMMENT{Any consistent ordering between elements of $\Omega$ is valid.}
		\STATE	Let $V^l$ be the elements assigned to machine $l$\;
		\STATE	Run \AlgStreamModified on each machine $l$ to obtain $\{S^l_\tau\}$ and $\{T^l_{\tau,i}\}$ for $1 \leq i \leq m$ and relevant values of $\tau$ on that machine
		\STATE	$l^*, \tau^* \gets \argmax_{l, \tau } \frac{1}{m} \sum_{i=1}^{m} f_i(T^l_{\tau,i})$
		\STATE	$S, \{T_i\} \gets \AlgRG(\bigcup_l \bigcup_\tau S^l_\tau)$
		\STATE {\bfseries Return:} $\argmax\{\frac{1}{m} \sum_{i=1}^{m} f_i(T_i), \frac{1}{m} \sum_{i=1}^{m} f_i(T^{l^*}_{{\tau^*}i})\}$
	\end{algorithmic}
\end{algorithm} 

\begin{theorem}\label{theorem:distributed-fast}
	The \AlgDistributedFast algorithm outputs sets $S^*, \{ T^*_i\} \subset S$, with $|S^*| \leq \ell, |T^*_i| \leq k$, such that
	\[ 
	\E[\frac{1}{m} \sum_{i=1}^m f_i(T^*_i)] \geq \frac{\alpha \cdot \gamma}{\alpha + \gamma} \cdot \Opt,
	\] 
	where $\alpha = \frac{1}{2} (1 - \frac{1}{\euler^2})$ and $\gamma = \frac{1}{6 + \epsilon}$. 
	The time complexity of algorithm is $O(\nicefrac{kmn \log \ell}{\mathsf{M}} + \mathsf{M} km\ell^2\log \ell )$.
\end{theorem}

The following lemma provides the equivalent of \cref{lemma:greedy} for \AlgStreamModified.
The rest of proof is exactly the same as the proof of \cref{theorem:distributed} with the only difference  that the approximation guarantee of \AlgStreamModified is $\gamma = \frac{1}{6 + \epsilon}$.

\begin{lemma} \label{lemma:greedy-fast}
	Let $A \subseteq \Omega$ and $B \subseteq \Omega$ be two disjoint subsets of $\Omega$. Suppose for each element $e \in B$, we have 
	\[\AlgStreamModified(A \cup \{e\}) = \AlgStreamModified(A).\] 
	Then we have:
	\[\AlgStreamModified(A \cup B) = \AlgStreamModified(A).\]
\end{lemma}
\begin{proof}	
	First note that because of the universal predefined ordering between elements of $\Omega$, the order of processing the elements would not change in different runs of \AlgStreamModified.
	Also, in the streaming setting, if an element $u^t$ changes the set of thresholds $\Tau^t$, then $u^t$ would be picked by those newly instantiated thresholds. 
	To show this, assume $\delta_{t-1} < \tau \leq \delta_t$ is one of the newly instantiated thresholds.
	For $\tau$, the sets $\{T_{\tau,i}\}$ are empty and we have:
	\[\tau \leq \sum_{i=1}^{m} \nabla_i(u^t| \emptyset) = \sum_{i=1}^{m} f_i(u^t) = \delta_t.\]
	Therefore, $u^t$ is added to all sets $\{T_{\tau,i} \}$.
	For an element $e \in B$, we have two cases: (i) $e$ has not changed the thresholds when it is arrived, or (ii) it has instantiated new thresholds (e.g., a new threshold $\tau$) but non of them is in the final thresholds  $\Tau^{n}$; because if $\tau \in \Tau^n$, then we have $e \in S_{\tau}^n$, and this contradicts with the definition of set $B$.
	
	Now consider $\AlgStreamModified(A \cup B)$. 
	We prove the lemma by contradiction. Assume
	\[ \AlgStreamModified(A \cup B) \neq \AlgStreamModified(A). \]
	Assume $e$ is the first element of $B$ which is picked by $\AlgStreamModified(A \cup B)$ for a threshold in $\Tau^n$.
	From the above, we know that non of the thresholds $\Tau^n$ of this running instance of the algorithm is instantiated when an element of $B$ is arrived. So, when $e$ is arrived, all the thresholds of $\Tau^n$ which are instantiated so far are from elements of $A$. 
	Also, since the order of processing of elements are fixed, $\AlgStreamModified(A \cup B)$ and $\AlgStreamModified(A \cup \{e\})$ would pick the same set of element till the point $e$ is arrived. If $e$ is picked by $\AlgStreamModified(A \cup B)$ for a threshold $\tau \in \Tau^n$, then $\AlgStreamModified(A \cup \{e\})$ would also pick $e$ for that threshold. This contradicts with the definition of set $B$.
\end{proof}

From \cref{theorem:distributed,theorem:distributed-fast}, we conclude that the optimum number of machines $\mathsf{M}$ for \AlgDistributed and \AlgDistributedFast is $O(\sqrt{\nicefrac{n}{\ell}})$ and $O(\nicefrac{\sqrt{n}}{\ell})$, respectively. 
Therefore, \AlgDistributedFast is a factor of $O(\nicefrac{\sqrt{n}}{\log \ell})$ and $O(\nicefrac{\sqrt{\ell}}{\log \ell})$ faster than \AlgRG and \AlgDistributed, respectively.
\vspace{-10pt}

\section{Applications} \label{section:apps}
In this section, we evaluate the performance of our algorithms in both the streaming and distributed settings.
We compare our work against several different baselines.
\subsection{Streaming Image Summarization} \label{subsection:imageStream}
In this experiment, we will use a subset of the \textit{VOC2012} dataset~\citep{voc2012}. This dataset has images containing objects from 20 different classes, ranging from birds to boats. For the purposes of this application, we will use $n = 756$ different images and we will consider all $m = 20$ classes that are available. 
Our goal is to choose a small subset $S$ of images that provides a good summary of the entire ground set $\Omega$. 
In general, it can be difficult to even define what a good summary of a set of images should look like. Fortunately, each image in this dataset comes with a human-labelled annotation that lists the number of objects from each class that appear in that image. 

Using the exemplar-based clustering approach \citep{mirzasoleiman2013distributed}, for each image we generate an $m$-dimensional vector $x$ such that $x_i$ represents the number of objects from class $i$ that appear in the image (an example is given in \cref{appendix:additional}).  
We define $\Omega_i$ to be the set of all images that contain objects from class $i$, and correspondingly $S_i = \Omega_i \cap S$ (i.e. the images we have selected that contain objects from class $i$).

We want to optimize the following monotone submodular functions: 
\[f_i(S) = L_i(\{e_0\}) - L_i(S \cup \{e_0\}),\]
where \[ L_i(S) = \frac{1}{| \Omega_i |} \sum_{x \in \Omega_i} \min_{y \in S_i} d(x,y).\]

We use $d(x,y)$ to denote the ``distance" between two images $x$ and $y$. More accurately, we measure the distance between two images as the $\ell_2$ norm between their characteristic vectors. We also use $e_0$ to denote some auxiliary element, which in our case is the all-zero vector. 

Since image data is generally quite storage-intensive, streaming algorithms can be particularly desirable. With this in mind, we will compare our streaming algorithm \textsc{Replacement-Streaming} against the non-streaming baseline of \textsc{Replacement-Greedy}. We also compare against a heuristic streaming baseline that we call \textsc{Stream-Sum}. This baseline first greedily optimizes the submodular function $F(S) = \sum^m_{i=1} f_i(S)$ using the streaming algorithm developed by \citet{buchbinder2015online}. Having selected $\ell$ elements from the stream, it then constructs each $T_i$ by greedily selecting $k$ of these elements for each $f_i$.

To evaluate the various algorithms, we consider two primary metrics: the objective value, which we define as $\sum^m_{i=1} f_i(T_i)$, and the wall-clock running time. We note that the trials were run using Python 2.7 on a quad-core Linux machine with 3.3 GHz Intel Core i5 processors and 8 GB of RAM. Figure \ref{streamGraphs} shows our results.

The graphs are organized so that each column shows the effects of varying a particular parameter, with the objective value being shown in the top row and the running time in the bottom row. 
The primary observation across all the graphs is that our streaming algorithm \textsc{Replacement-Streaming} not only achieves an objective value that is similar to that of the non-streaming baseline \textsc{Replacement-Greedy}, but it also speeds up the running time by a full order of magnitude. We also see that \textsc{Replacement-Streaming} outperforms the streaming baseline \textsc{Stream-Sum} in both objective value and running time.
 
 Another noteworthy observation from Figure~\ref{streamGraphs}(c) is that $\epsilon$ can be increased all the way up to $\epsilon = 0.5$ before we start to see loss in the objective value. Recall that $\epsilon$ is the parameter that trades off the accuracy of  \textsc{Replacement-Streaming} with the running time by changing the granularity of our guesses for $\Opt$. As seen Figure \ref{streamGraphs}(f), increasing $\epsilon$ up to 0.5 also covers the majority of running time speed-up, with diminishing returns kicking in as we get close to $\epsilon = 1$.

 Also in the context of running time, we see in Figure \ref{streamGraphs}(e) that \textsc{Replacement-Streaming} actually speeds up as $k$ increases. This seems counter-intuitive at first glance, but one possible reason is that the majority of the time cost for these replacement-based algorithms comes from the swapping that must be done when the $T_i$'s fill up. Therefore, the longer each $T_i$ is not completely full, the faster the overall algorithm will run. 

 \begin{figure*}[htb!]
	\centering
	\subfloat[$k = 5$, $\epsilon = 0.5$]{\includegraphics[height=1.33in]{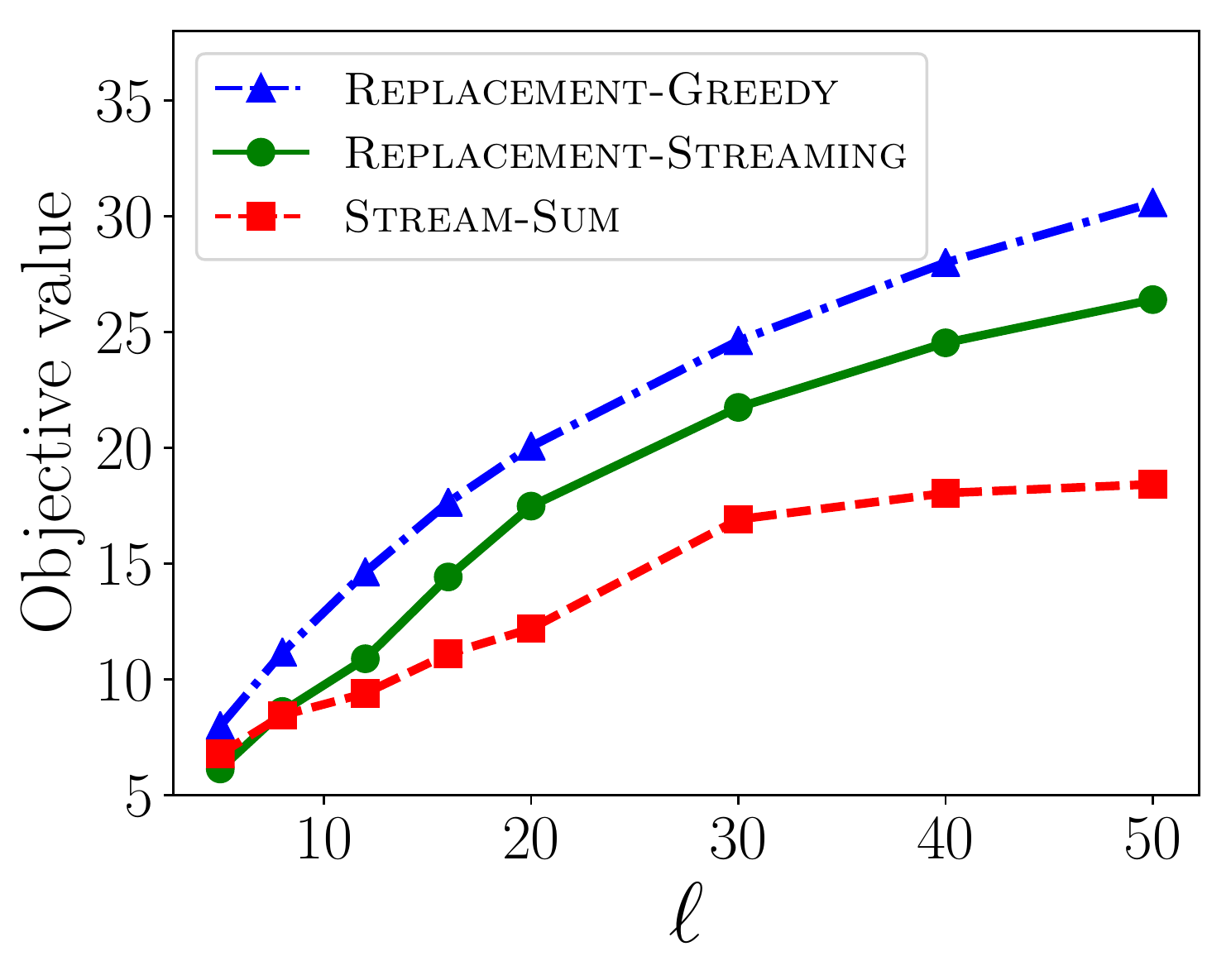} \label{fig:lv} }
	\hspace{0.025in}
	\subfloat[$\ell = 25$, $\epsilon = 0.5$]{\includegraphics[height=1.33in]{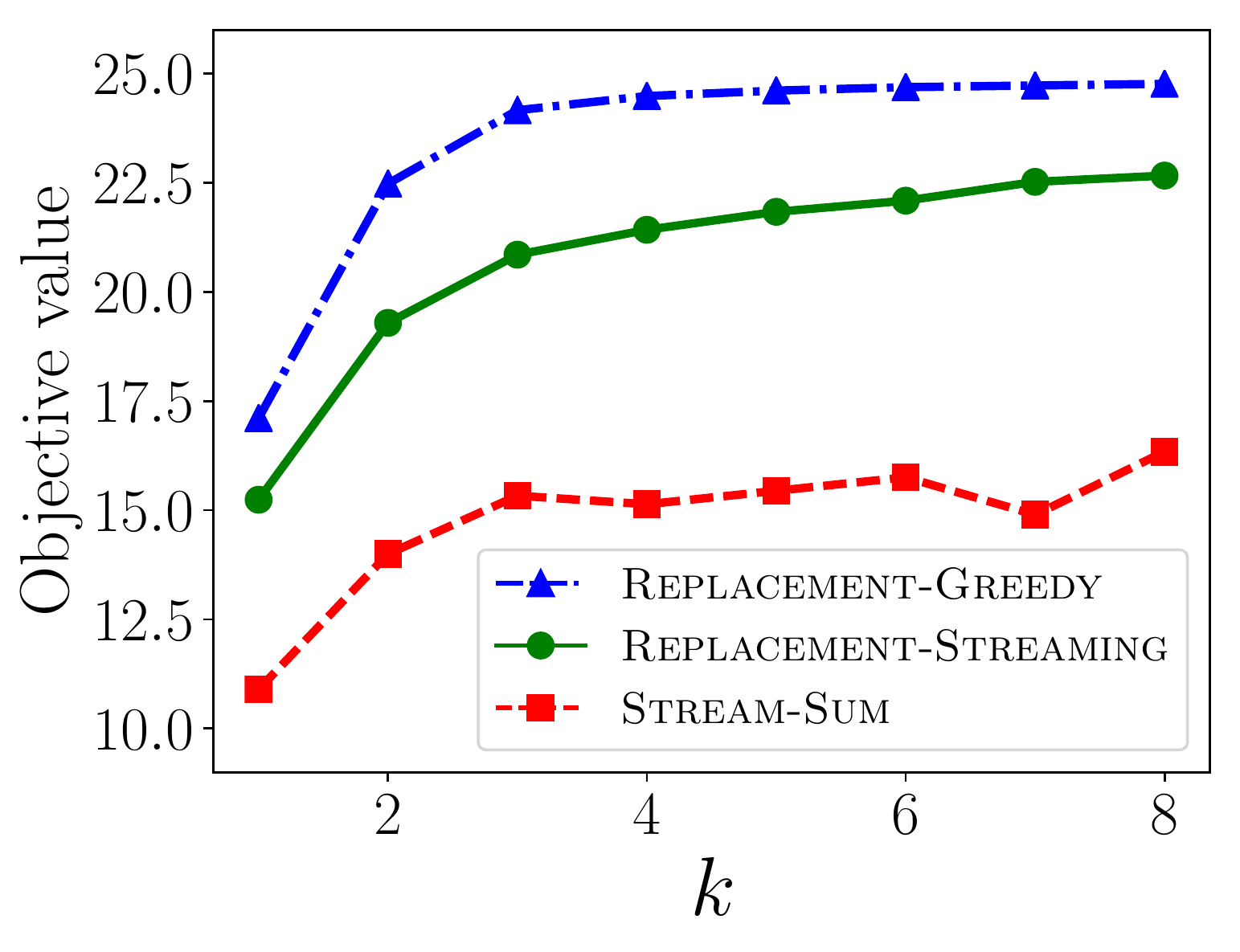}\label{fig:kv}}
	\hspace{0.025in}
	\subfloat[$\ell = 25$, $k = 5$]{\includegraphics[height=1.3in]{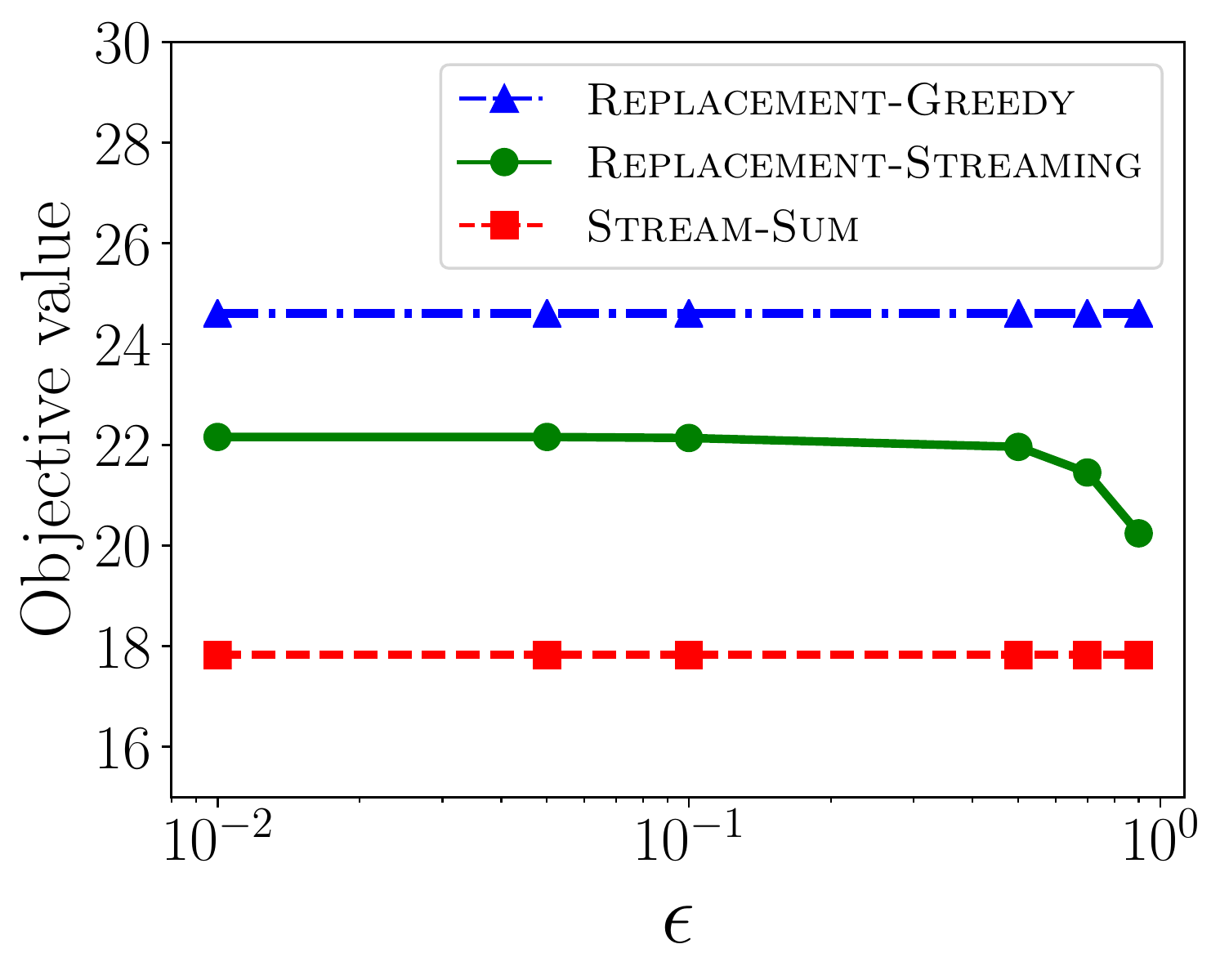}\label{fig:ev}}

	\subfloat[$k = 5$, $\epsilon = 0.5$]{\includegraphics[height=1.33in]{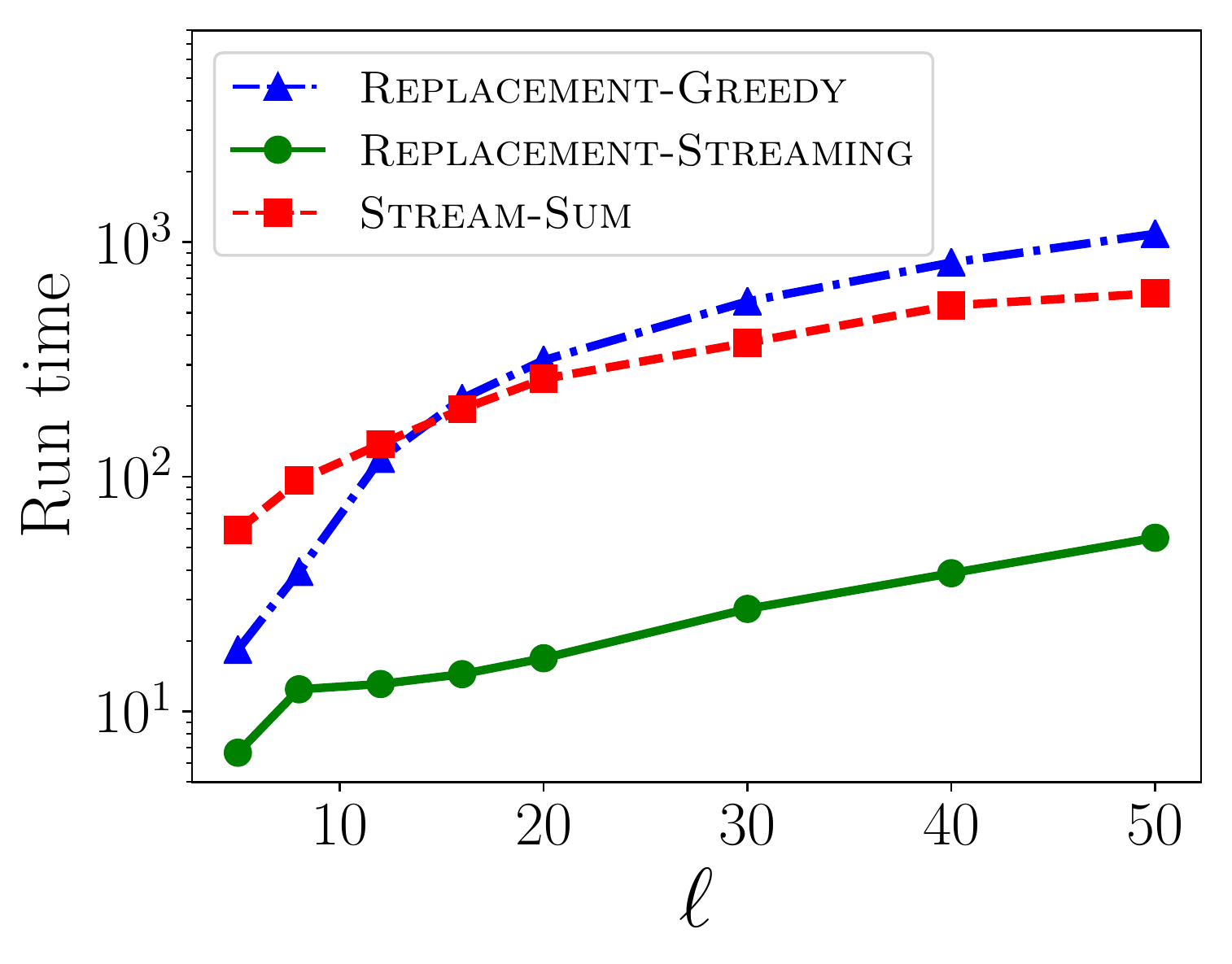}\label{fig:lt}}
	\hspace{0.025in}
	\subfloat[$\ell = 25$, $\epsilon = 0.5$]{\includegraphics[height=1.33in]{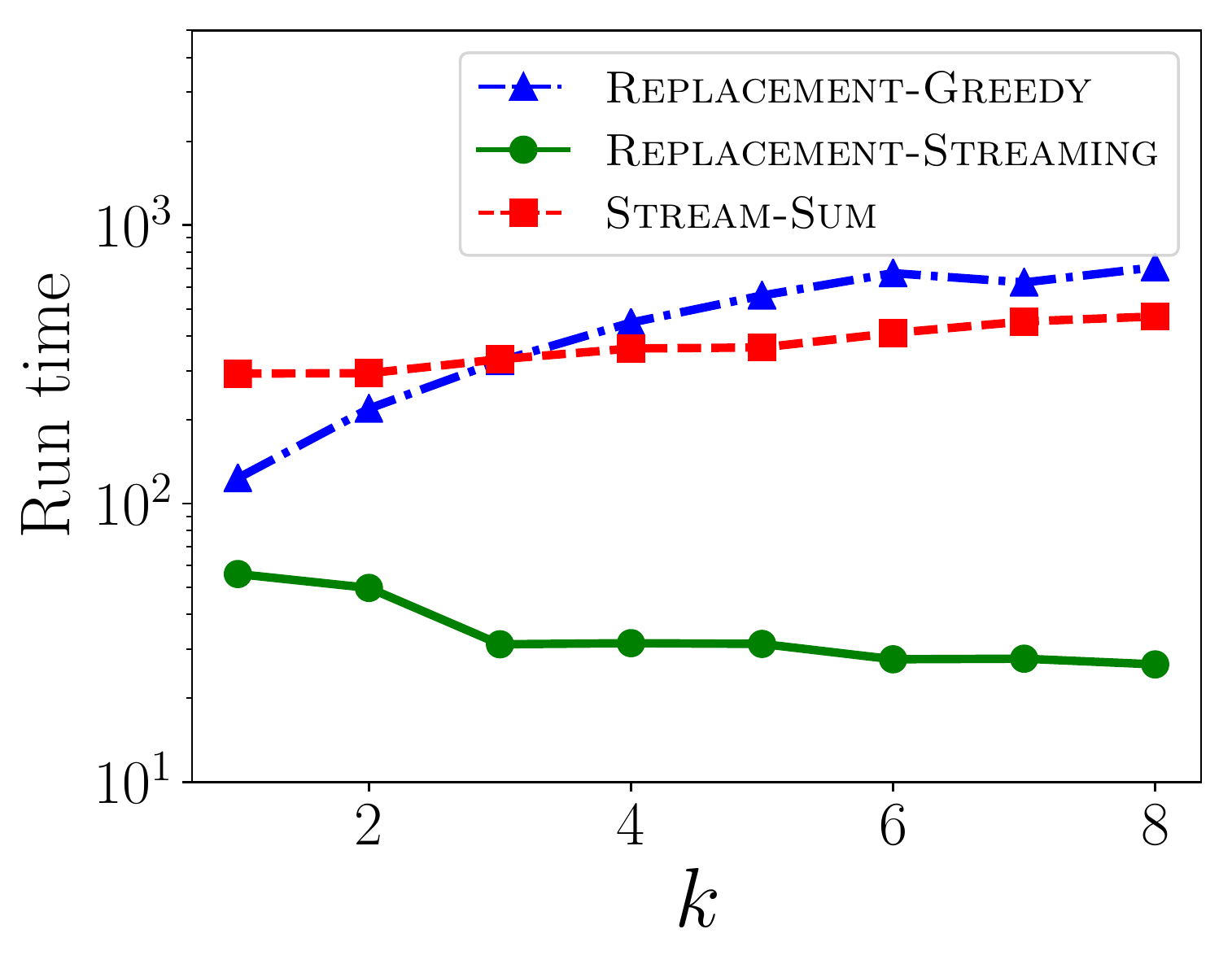}\label{fig:kt}}
	\hspace{0.025in}
	\subfloat[$\ell = 25$, $k = 5$]{\includegraphics[height=1.33in]{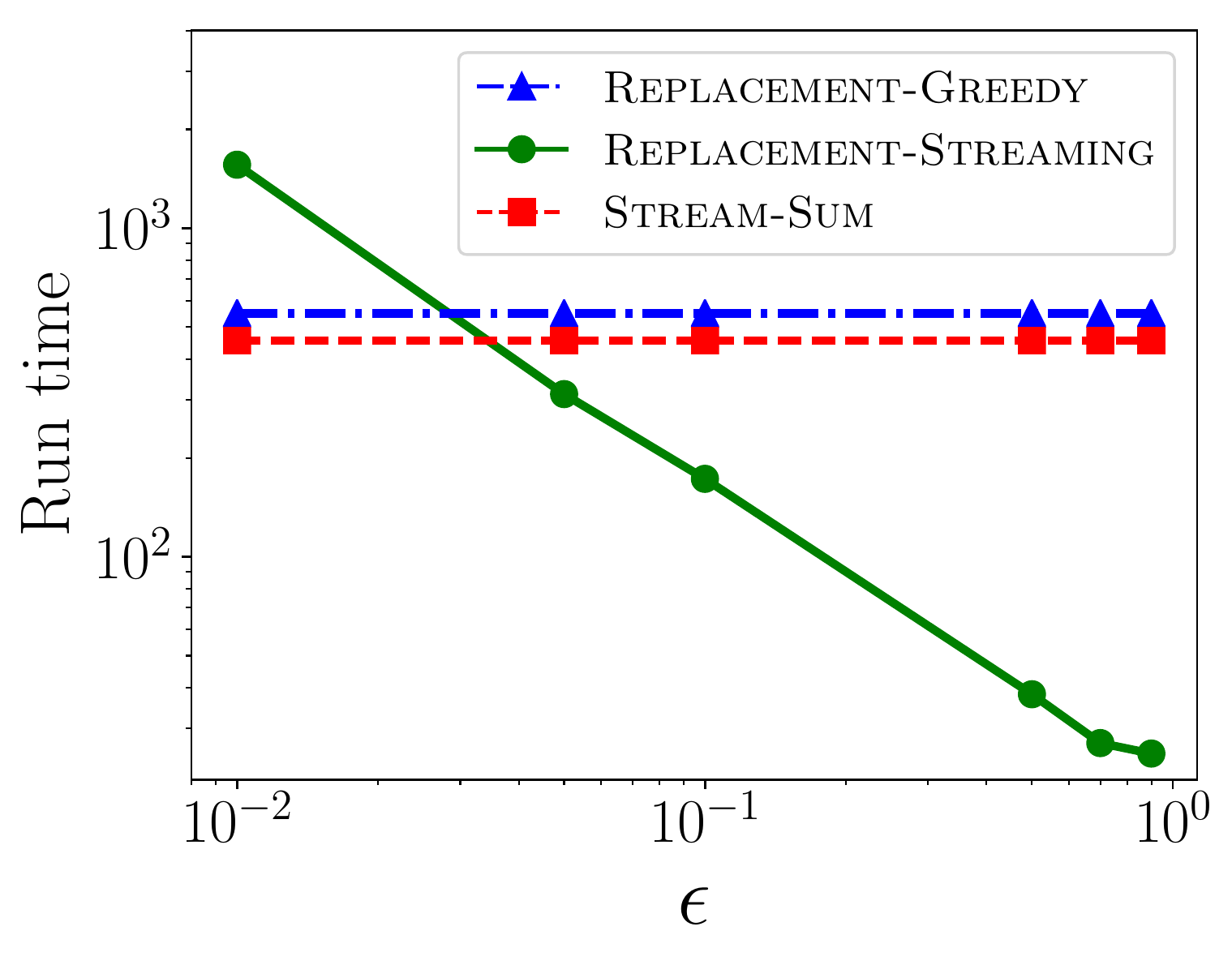}\label{fig:et}}
	\caption{The top row of graphs shows the objective values achieved by the various algorithms, while the bottom row shows the run times. In (a) and (d) we vary $l$, the maximum size of the subset $S$. In (b) and (e), we vary $k$, the maximum size of the set $T_i$ assigned to each function $f_i$. Lastly, in (c) and (f), we vary $\epsilon$, the parameter that controls the number of guesses we make for OPT. Our streaming algorithm (shown in green) compares favorably with the non-streaming version (shown in blue) in terms of objective value, and outperforms it significantly in terms of runtime.}
	\label{streamGraphs}
\end{figure*}

Figure \ref{fig:samples} shows some sample images selected by \textsc{Replacement-Greedy}  (top) and  \textsc{Replacement-Streaming} (bottom). Although the two summaries contain only one image that is exactly the same, we see that the different images still have a similar theme. For example, both images in the second column contain bikes and people; while in the third column, both images contain sheep.

 \begin{figure}[htb!]
	\centering
	\includegraphics[width=0.95\textwidth]{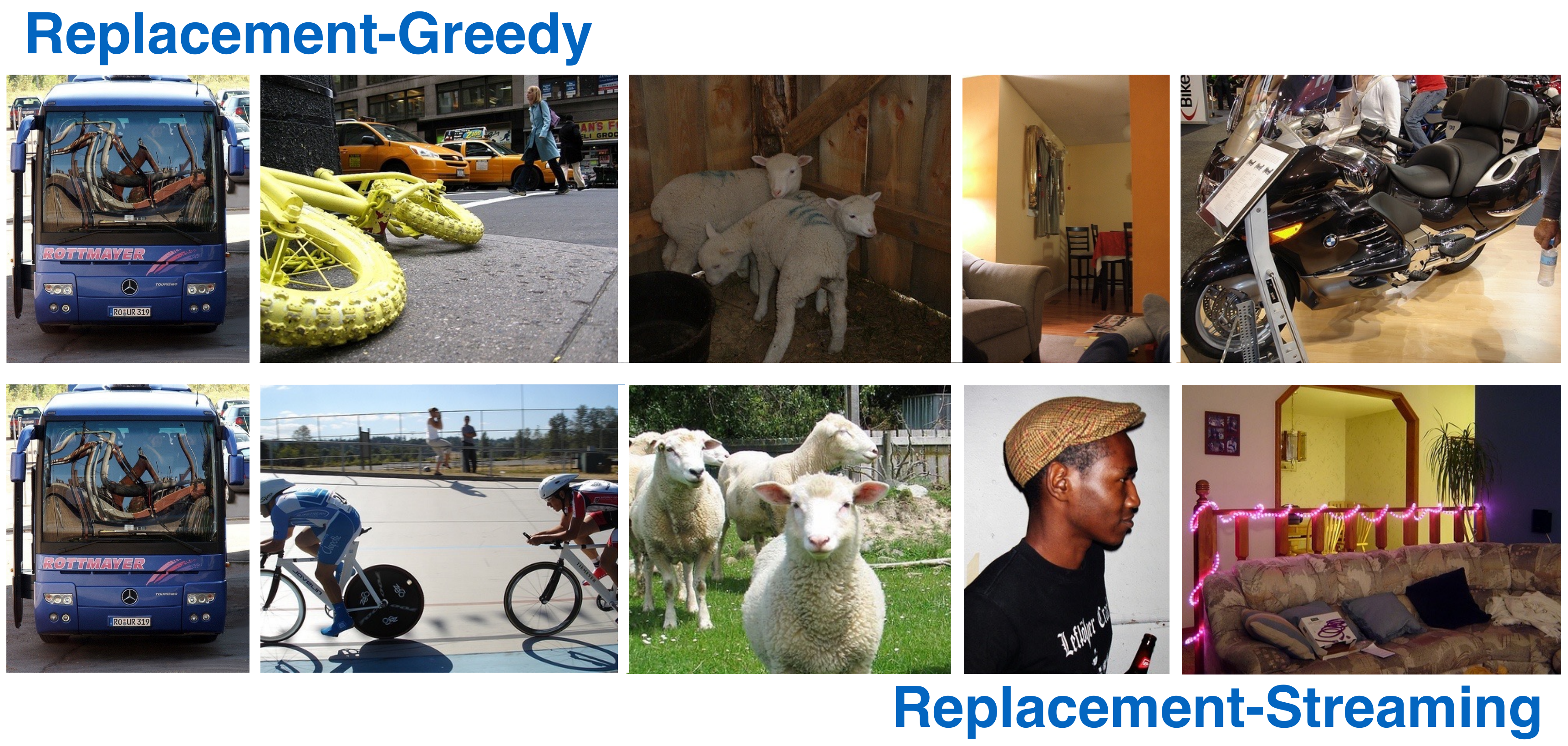}
	\caption{Representative images selected in the different settings. Although \AlgStream (bottom row) does not select all the same images as \AlgRG (top row), we see that the different images still have a similar theme. For example, both images in the second column contain bikes and people; while in the third column, both images contain sheep.}
	\label{fig:samples}
\end{figure}

\subsection{Distributed Ride-Share Optimization}

\begin{figure*}[htb!]
\centering
\subfloat[]{\includegraphics[height=1.4in]{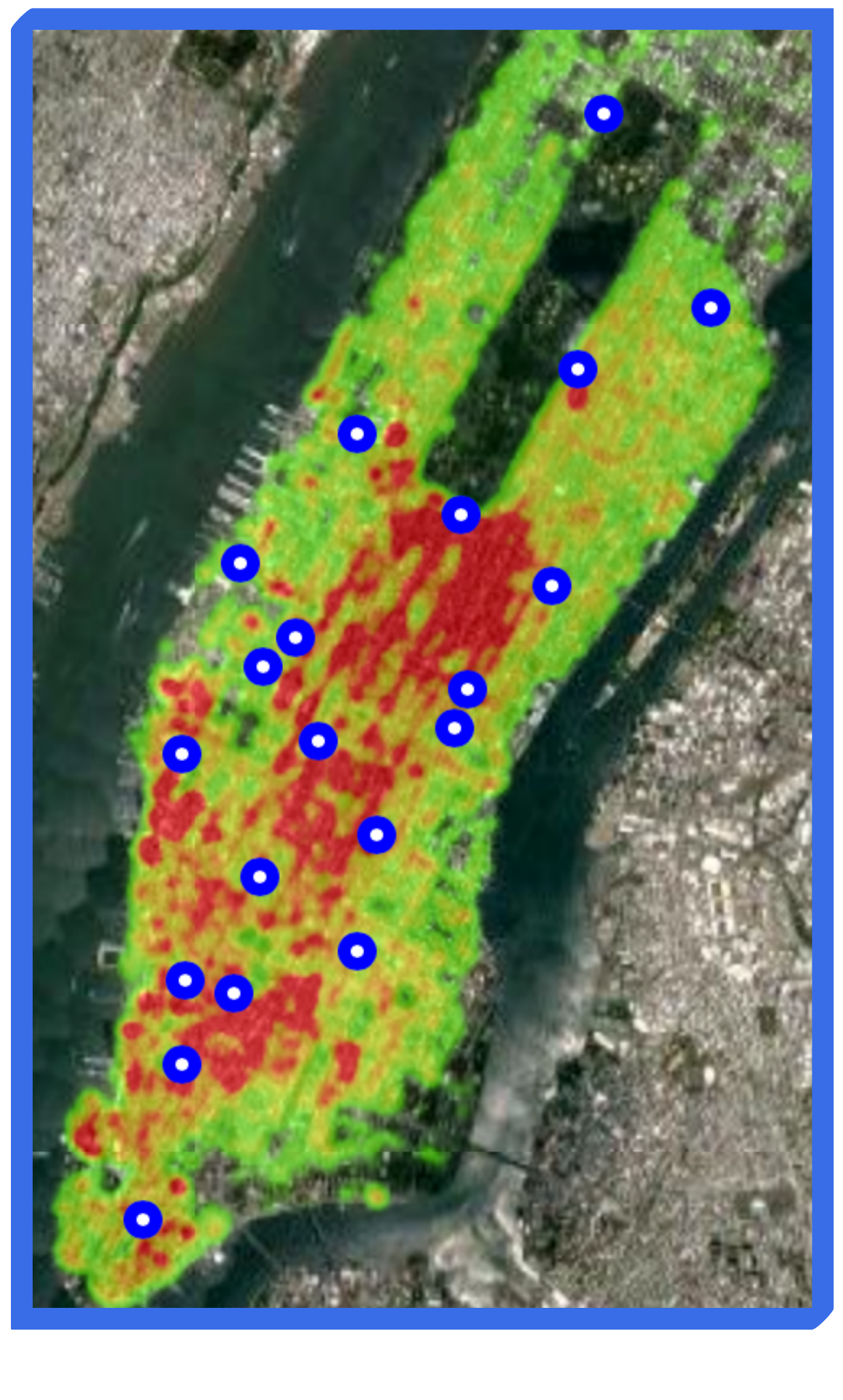} \label{fig:h1} }
\hspace{0.10in}
\subfloat[$\ell = 30$, $k = 3$]{\includegraphics[height=1.4in]{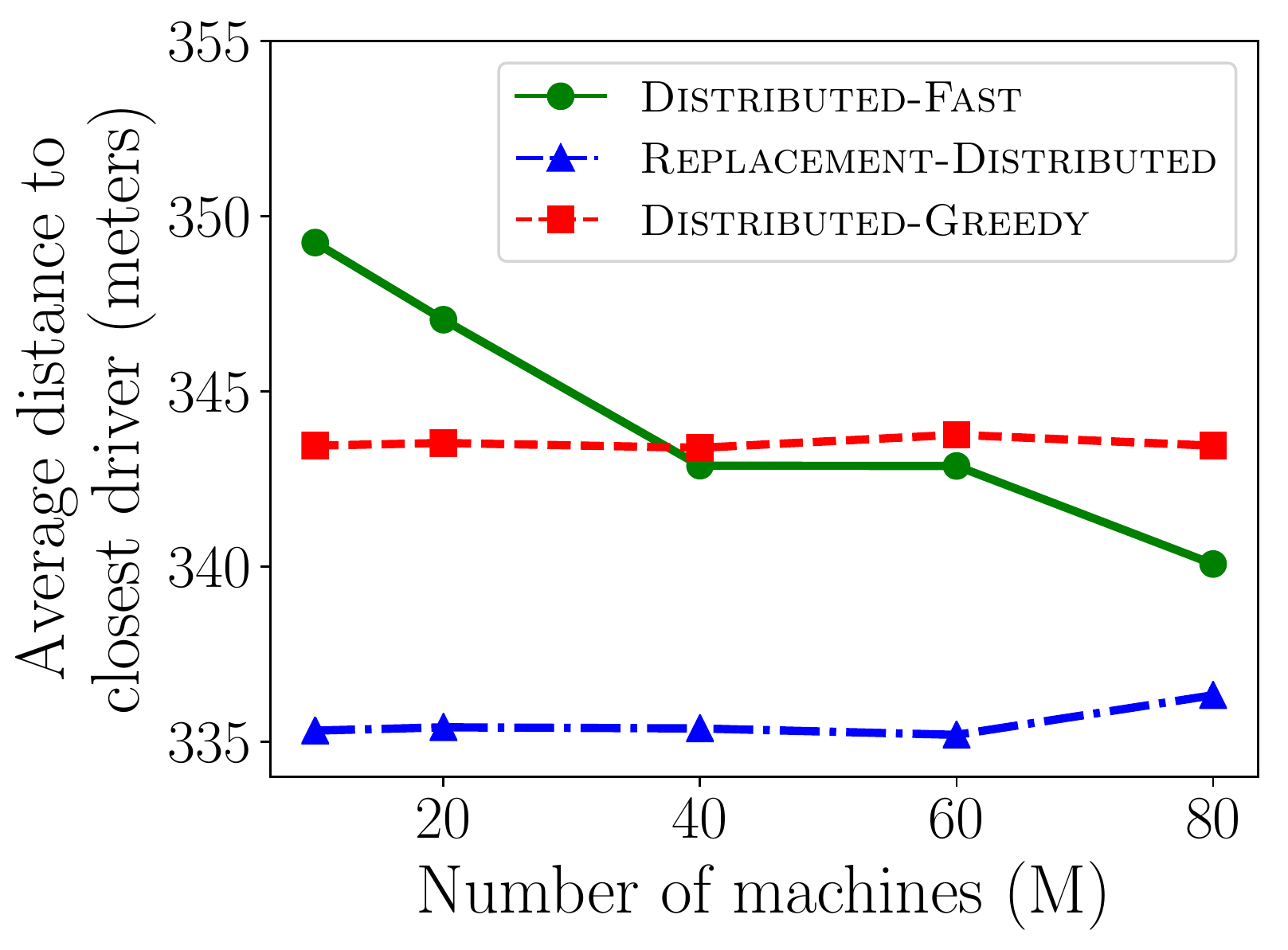}\label{fig:od}}
\hspace{0.10in}
\subfloat[$\ell = 30$, $k = 3$]{\includegraphics[height=1.4in]{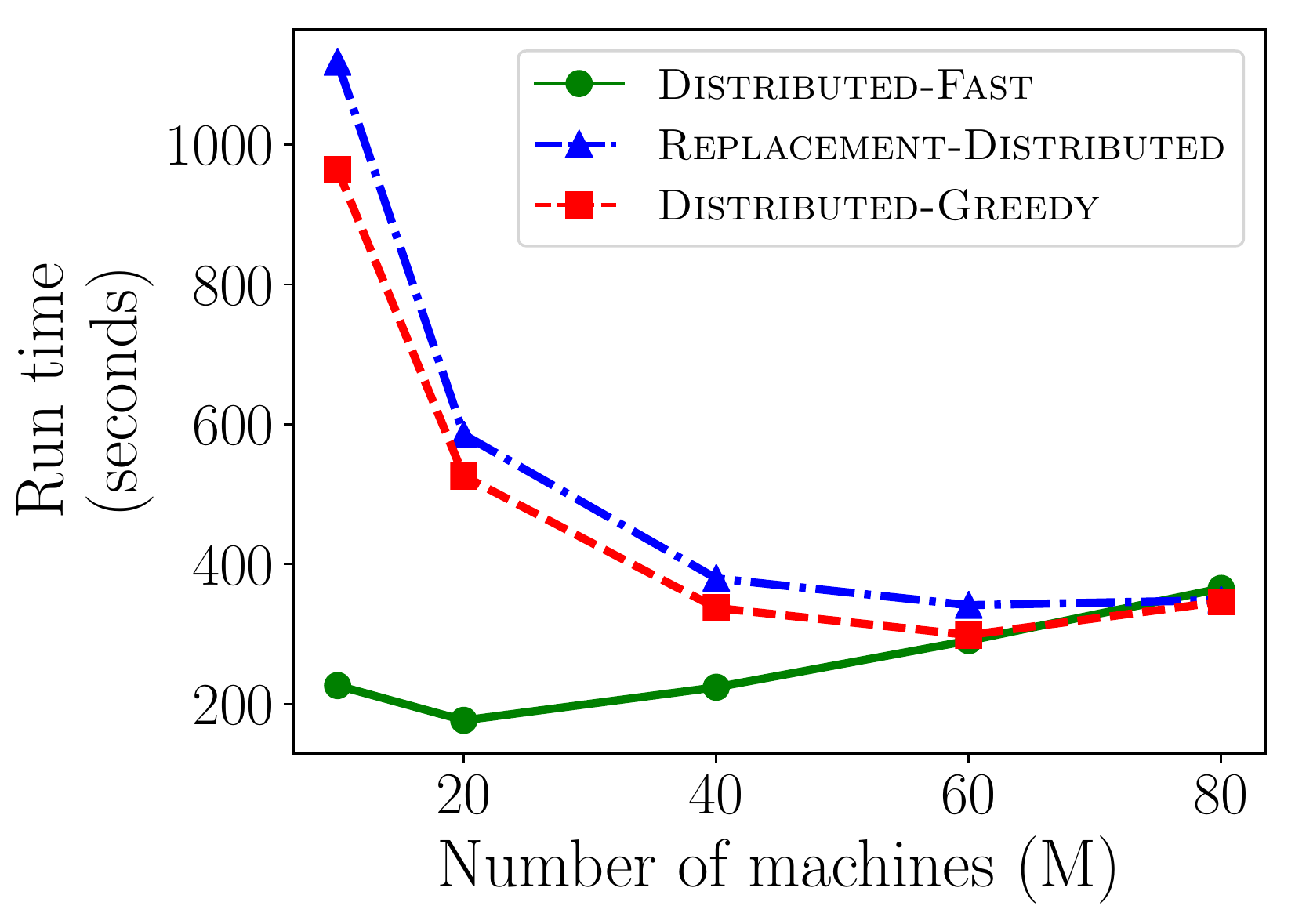}\label{fig:ot}}

\subfloat[]{\includegraphics[height=1.4in]{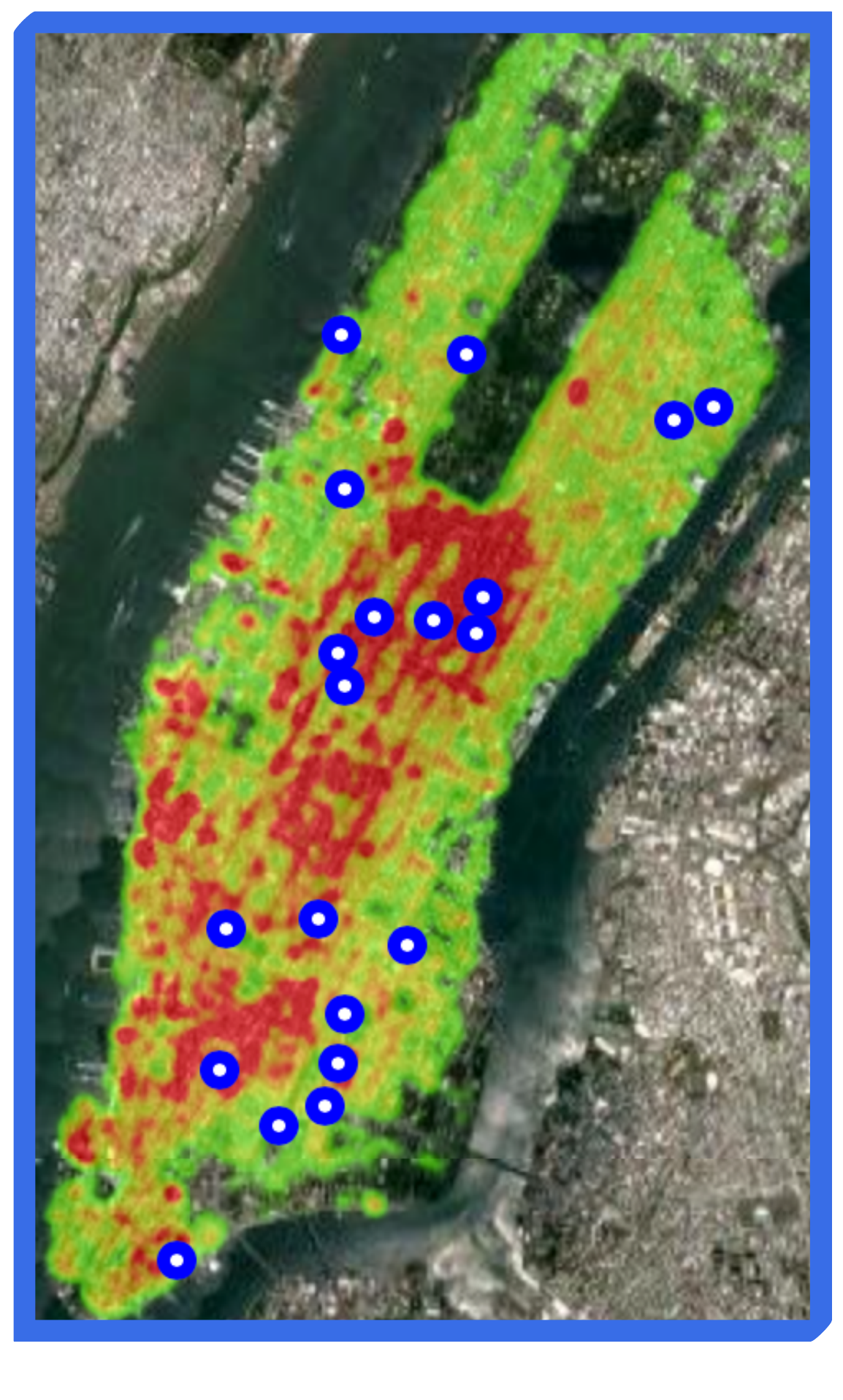}\label{fig:h2}}
\hspace{0.10in}
\subfloat[$k=3$]{\includegraphics[height=1.4in]{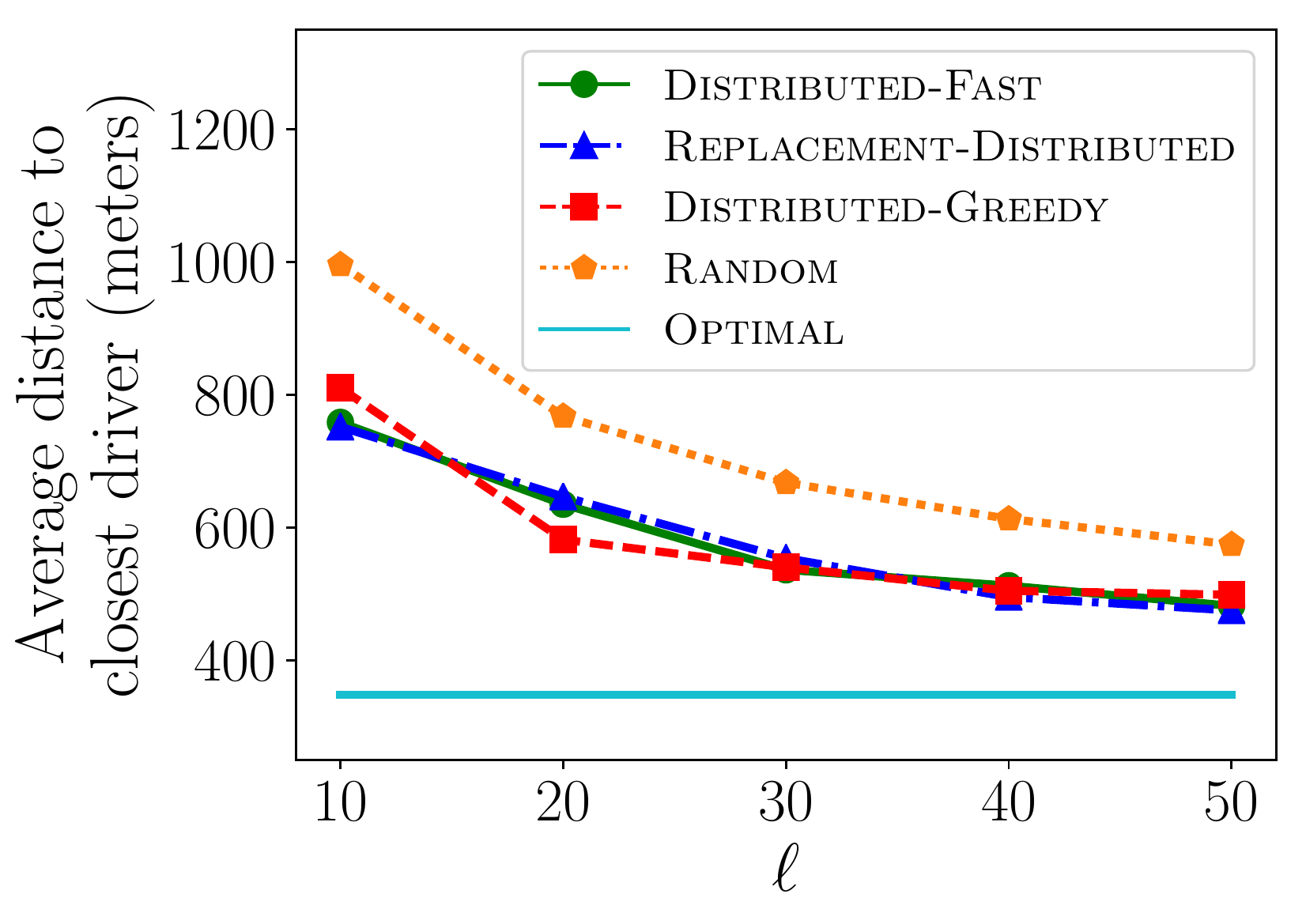}\label{fig:nd}}
\hspace{0.10in}
\subfloat[$k=3$]{\includegraphics[height=1.4in]{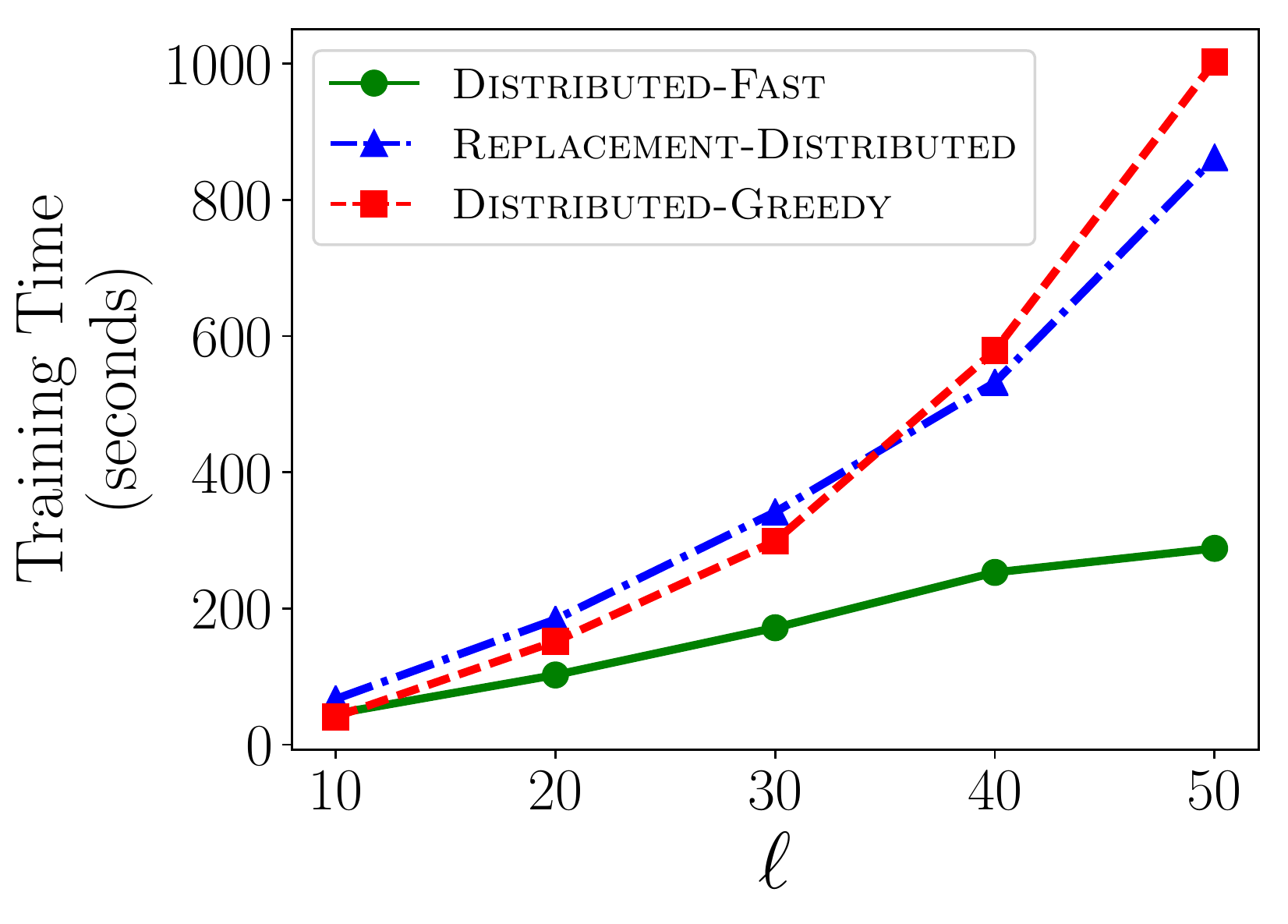}\label{fig:nt}}
\caption{(a) shows a heatmap of all pick-up locations, as well as the centers of the twenty random regions that define each function $f_i$. (b) and (c) show the effects of changing the number of machines we use to distribute the computation. (d) shows the centers of the twenty new regions (chosen from the same distribution) used for the evaluation in (e). (f) shows the training time for each summary used in (e).}
\label{fig:distGraphs}
\end{figure*}

In this application we want to use past Uber data to select optimal waiting locations for idle drivers. Towards this end, we analyze a dataset of 100,000 Uber pick-ups in Manhattan from September 2014 \citep{UberDataset}, where each entry in the dataset is given as a (latitude, longitude) coordinate pair. We model this problem as a classical facility location problem, which is known to be monotone submodular.

Given a set of potential waiting locations for drivers, we want to pick a subset of these locations so that the distance from each customer to his closest driver is minimized. In particular, given a customer location $a = (x_a,y_a)$, and a waiting driver location $b = (x_b, y_b)$, we define a ``convenience score" $c(a,b)$ as follows: $c(a,b) = 2 - \frac{2}{1 + e^{-200d(a,b)}}$,
where $\d(a,b) = |x_a - x_b| + |y_a - y_b|$ is the Manhattan distance between the two points.

Next, we need to introduce some functions we want to maximize. For this experiment, we can think about different functions corresponding to different (possibly overlapping) regions around Manhattan. The overlap means that there will still be some inherent connection between the functions, but they are still relatively distinct from each other. More specifically, we construct regions $R_1, \hdots, R_m$ by randomly picking $m$ points across Manhattan. Then, for each point $p_i$, we want to define the corresponding region $R_i$ by all the pick-ups that have occurred within one kilometer of $p_i$. However, to keep the problem computationally tractable, we instead randomly select only ten pick-up locations within that same radius. Figure \ref{fig:distGraphs}(a) shows the center points of the $m=20$ randomly selected regions, overlaid on top of a heat map of all the customer pick-up locations.
 
Given any set of driver waiting locations $T_i$, we define $f_i(T_i)$ as follows:
 $f_i(T_i) = \sum_{a \in R_i} \max_{b \in T_i} \hspace{0.05cm} c(a,b).$
For this application, we will use every customer pick-up location as a potential waiting location for a driver, meaning we have 100,000 elements in our ground set $\Omega$. This large number of elements, combined with the fact that each single function evaluation is computationally intensive, means running the regular \textsc{Replacement-Greedy} will be prohibitively expensive. Hence, we will use this setup to evaluate the two distributed algorithms we presented in Section \ref{section:distributed}. We will also compare our algorithms against a heuristic baseline that we call \textsc{Distributed-Greedy}. This baseline will first select $\ell$ elements using the greedy distributed framework introduced by \citet{mirzasoleiman2013distributed}, and then greedily optimize each $f_i$ over these $\ell$ elements. 

Each algorithm produces two outputs: a small subset $S$ of potential waiting locations (with size $\ell=30$), as well as a solution $T_i$ (of size $k=3$) for each function $f_i$. In other words, each algorithm will reduce the number of potential waiting locations from 100,000 to 30, and then choose 3 different waiting locations for drivers in each region. 

In Figure \ref{fig:distGraphs}(b), we graph the average distance from each customer to his closest driver, which we will refer to as the cost. One interesting observation is that while the cost of \textsc{Distributed-Fast} decreases with the number of machines, the costs of the other two algorithms stay relatively constant, with \textsc{Replacement-Distributed} marginally outperforming \textsc{Distributed-Greedy}. In Figure \ref{fig:distGraphs}(c), we graph the run time of each algorithm. We see that the algorithms achieve their optimal speeds at different values of $M$, verifying the theory at the end of \cref{section:distributed}. Overall, we see that while all three algorithms have very comparable costs, \textsc{Distributed-Fast} is significantly faster than the others.

While in the previous application we only looked at the objective value for the given functions $f_1, \hdots, f_m$, in this experiment we also evaluate the utility of our summary on new functions drawn from the same distribution. That is, using the regions shown in Figure \ref{fig:distGraphs}(a), each algorithm will select a subset $S$ of potential waiting locations. Using only these reduced subsets, we then greedily select $k$ waiting locations for each of the twenty new regions shown in \ref{fig:distGraphs}(d). 

In Figure \ref{fig:distGraphs}(e), we see that the summaries from all three algorithms achieve a similar cost, which is significantly better than \textsc{random}. In this scenario, \textsc{random} is defined as the cost achieved when optimizing over a random size $\ell$ subset and \textsc{optimal} is defined as the cost that is achieved when optimizing the functions over the entire ground set rather than a reduced subset. In Figure  \ref{fig:distGraphs}(f), we confirm that \textsc{Distributed-Fast} is indeed the fastest algorithm for constructing each summary. Note that \ref{fig:distGraphs}(f) is demonstrating how long each algorithm takes to construct a size $\ell$ summary, not how long it is taking to optimize over this summary.
\section{Conclusion} \label{section:conc}
\vspace{-5pt}
To satisfy the need for scalable data summarization algorithms, this paper focused on the two-stage submodular maximization framework and provided the first streaming and distributed solutions to this problem. In addition to constant factor theoretical guarantees, we demonstrated the effectiveness of our algorithms on real world applications in image summarization and ride-share optimization.
\paragraph{Acknowledgement} Amin Karbasi was supported by DARPA Young Faculty Award (D16AP00046) and AFOSR Young Investigator Award (FA9550-18-1-0160).
Ehsan Kazemi was supported by Swiss National Science Foundation (Early Postdoc.Mobility) under grant
number 168574. 

\bibliographystyle{plainnat}
\bibliography{./tex/references}

\newpage
\appendix
\section{Summary of Notations} \label{section:terms}
\cref{terms} provides the  summary of notations used in this paper.
\begin{table}[h]
	\centering
	\caption{Summary of important terminology } \label{terms}
		\scalebox{0.75}{
	\begin{tabular}{ccl}
		\toprule
		\textbf{Set} & \textbf{Cardinality} & \hspace{5.17cm} \textbf{Description} \\ 
		\midrule
		$F$ & $m$ & Set of functions $(f_1, \hdots, f_m)$ drawn from an unknown distribution $\mathbb{D}$ of monotone submodular functions.  \\ 
		\midrule
		$\Omega$ & $n$ & Given ground set of all elements. Generally so large that even greedy is too expensive. \\ 
		\midrule
			$S^{m,\ell}$ & $\ell$ & The optimum solution to \cref{eq:problem-m}, i.e., $S^{m, \ell} = \argmax_{S \subseteq \Omega, |S| \leq \ell} \frac{1}{m} \sum_{i=1}^{m} \max_{|T| \leq k, T\subseteq S} f_i(T)$.  \\ 
				\midrule
			$S^{m,\ell}_i$ & $k$ &  The optimum solution to each function $f_i$  from set $S^{m,\ell}$, i.e., $S^{m,\ell}_{i} = \argmax_{S \subseteq S^{m,\ell} , |S| \leq k} f_i(S)$.   \\ 
		\midrule
			\Opt  & 1&  The value of optimum solution to \cref{eq:problem-m}, i.e., $\Opt = \frac{1}{m} \sum_{i=1}^{m} f_i(S_i^{m,\ell})$.\\
			\midrule
		$S$ & $\ell$ & Reduced subset of elements we want to select. Ideally sublinear in $n$, but still representative. \\ 
		\midrule
		$T_i$ & $k$ & Solution we select for each function $f_i$ (chosen from $S$), i.e., $T_i \subset S$. \\ 
		\bottomrule
	\end{tabular}
}
\end{table}

\section{\AlgRG} \label{section-replace-greedy}
In this section, in order to make the current manuscript self-contained, we describe the \AlgRG from \citep{stan17}. We use this greedy algorithm in \cref{section:distributed} as one of the building blocks of our distributed algorithms.

We first define few necessary notations.
The additive value of an element $x$ to a set $A$ from a function $f_i$ is defined as follows:
\[ \Lambda_i(x, A) = \left\{
\begin{array}{ll}
 f_i(x|A)& \mbox{if } |A| < k , \\
\max\{0, \Delta_i(x, A)\}& \mbox{o.w.},
\end{array}
\right.\]
where $\Delta_i(x,A)$ is defined in \cref{eq:delta}. We also define:
\[ \RepG_i(x, A) = \left\{
\begin{array}{ll}
\emptyset & \mbox{if } |A| < k , \\
\emptyset & \Delta_i(x,A) < 0, \\
\Rep_i(x, A)	& \mbox{o.w.},
\end{array}
\right.\]
where $\Rep_i(x, A)$ is defined in \cref{eq:rep}. 
Indeed, $ \RepG_i(x, A) $ represents the element from set $A$ which should be replaced with $x$ in order to get the maximum (positive) additive gain, where the cardinality constraint $k$ is satisfied. 
\AlgRG starts with empty sets $S$ and $\{T_i\}$. In $\ell$ rounds, it greedily adds elements with the maximum additive gains $\sum_{i=1} ^m \Lambda_i(x, T_i)$ to set $S$. If the gain of adding these elements (or exchanging with one element of $T_i$ where there exists $k$ elements in $T_i$) is non-negative, we also update sets $T_i$.
\AlgRG is outlined in \cref{alg:replacement_greedy}.
\begin{algorithm}[H]
	\caption{\AlgRG}
	\label{alg:replacement_greedy}
	\begin{algorithmic}[1]
	\STATE $S \gets \emptyset$ and $T_i \gets \emptyset$ for all $1 \leq i \leq m$
	\FOR{$1 \leq j \leq \ell$}
	\STATE $x^* \geq \argmax_{x \in \Omega} \sum_{i=1}^{m} \Lambda_i(x, T_i)$
	\STATE $S \gets S + x^*$
	\FOR{$1 \leq i \leq m$}
	\IF{$\Lambda_i(x^*, T_i) > 0$}
	\STATE $T_i \gets T_i  + x^* - \RepG_i(x^*, T_i)$
	\ENDIF
	\ENDFOR
	\ENDFOR
		\STATE {\bfseries Return:} $S$ and $\{T_i\}$ 
	\end{algorithmic}
\end{algorithm} 
\section{VOC2012 Feature Explanation} \label{appendix:additional}

To further clarify the VOC2012 dataset used in \cref{subsection:imageStream}, we explicitly list the twenty classes that appear in the dataset. We also give an example of an image from the dataset and its corresponding characteristic vector.

\begin{figure}[h]
\centering
\subfloat[]{\includegraphics[height=2.1in]{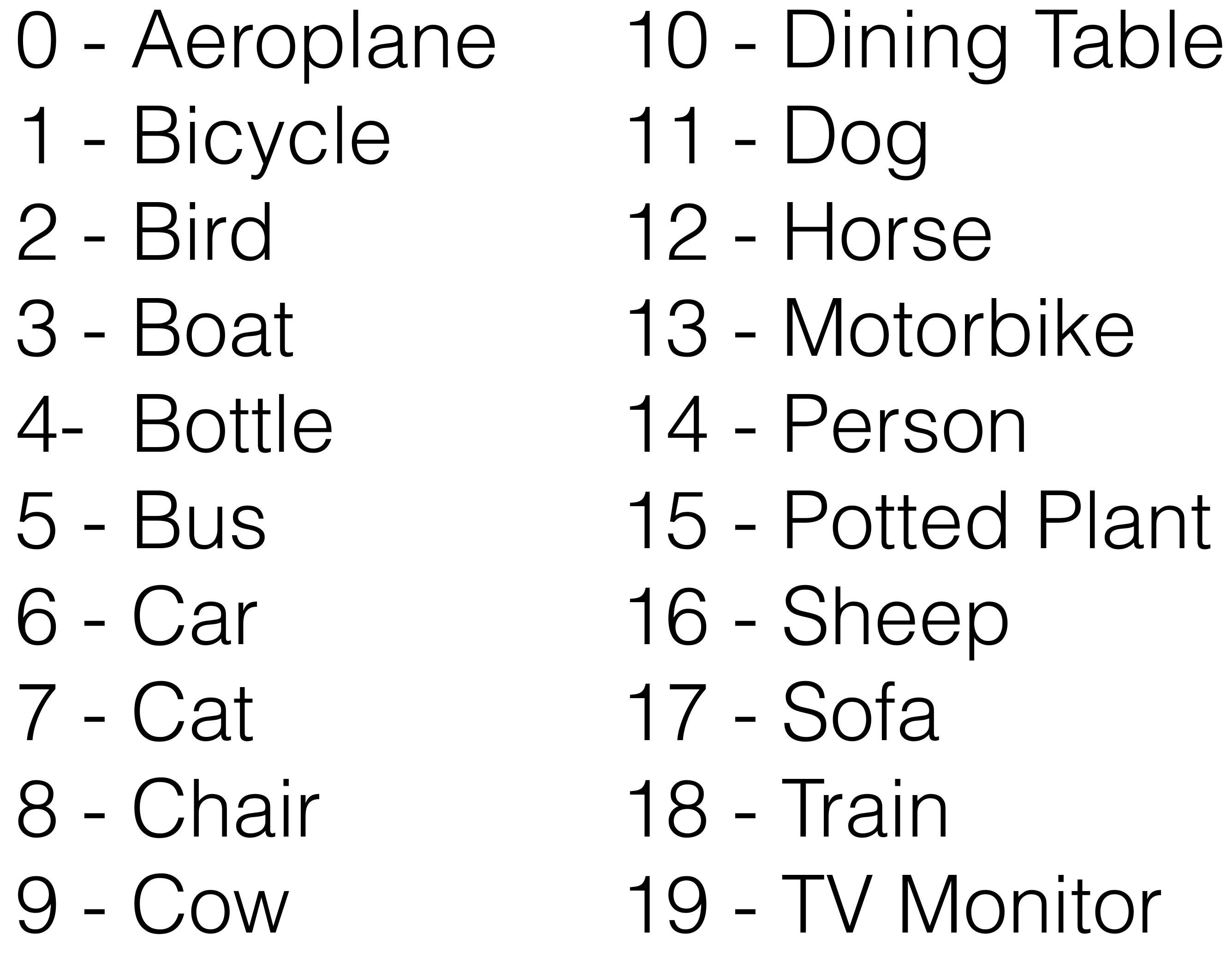}\label{fig:classes}}
\hspace{0.1in}
\subfloat[]{\includegraphics[height=2.1in]{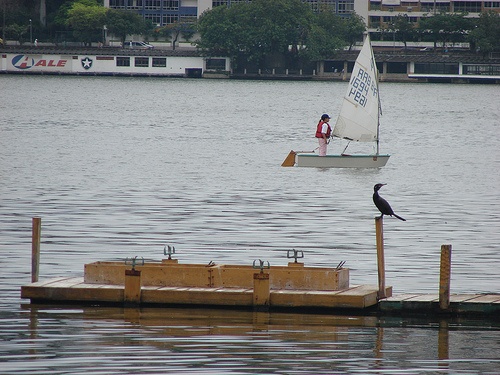} \label{fig:img1} }
\caption{(a) shows the twenty classes that appear in the \textit{VOC2012} dataset. The number adjacent to each class represents the index of that class in the characteristic vector associated with each image. For example, the image shown in (b) contains one boat, one bird, and one person. Therefore, the characteristic vector for this image is [0, 0, 1, 1, 0, 0, 0, 0, 0, 0, 0, 0, 0, 0, 1, 0, 0, 0, 0, 0]. This also means that the image in (b) appears in the sets $\Omega_2$,  $\Omega_4$, and  $\Omega_{14}$.}
\label{classes}
\end{figure}

\end{document}